%% file: paper.tex
\documentclass[]{jingdong}
\usepackage[toc,page,header]{appendix}

\input{common}
\definecolor{lightred}{rgb}{1.0, 0.9, 0.9}

\usepackage{cleveref}
\theoremstyle{plain}
\newtheorem{theorem}{Theorem}[section]
\newtheorem{proposition}[theorem]{Proposition}

\theoremstyle{definition}

\theoremstyle{remark}
\usepackage{subcaption}
\newtheorem{remark}[theorem]{Remark}
\usepackage[textsize=tiny]{todonotes}
\usepackage{fontawesome}

\usepackage{graphicx}
\usepackage[edges]{forest}

\usepackage{algorithm}
\usepackage{algpseudocode}
\usepackage{float}
\usepackage{pgfplots}
\pgfplotsset{compat=1.18}

\usepackage{colortbl}
\definecolor{badred}{HTML}{FDDEDE}
\definecolor{goodgreen}{HTML}{E0F0E0}
\definecolor{headblue}{HTML}{E8EEF8}

\usepackage{array}

\tcbset{
  aibox/.style={
    width=\linewidth,
    top=8pt,
    bottom=4pt,
    colback=brown!5!white,
    colframe=brown!55!black,
    colbacktitle=brown!55!black,
    fonttitle=\bfseries,
    enhanced,
  }
}
\newcolumntype{C}[1]{>{\centering\let\newline\\\arraybackslash\hspace{0pt}}m{#1}}
\newtcolorbox{AIbox}[2][]{aibox,title=#2,#1}

\tcbset{
  rubriccasebox/.style={
    width=\linewidth,
    top=8pt,
    bottom=4pt,
    colback=orange!6!white,
    colframe=orange!65!black,
    colbacktitle=orange!65!black,
    fonttitle=\bfseries,
    enhanced,
  }
}
\newtcolorbox{RubricCaseBox}[2][]{rubriccasebox,title=#2,#1}

\tcbset{
  casebox/.style={
    width=\linewidth,
    top=8pt,
    bottom=4pt,
    colback=black!3!white,
    colframe=black!50,
    colbacktitle=black!50,
    fonttitle=\bfseries,
    enhanced,
  }
}
\newtcolorbox{CaseStudyBox}[2][]{casebox,title=#2,#1}

\tcbset{
  failcasebox/.style={
    width=\linewidth,
    top=8pt,
    bottom=4pt,
    colback=violet!5!white,
    colframe=violet!55!black,
    colbacktitle=violet!55!black,
    fonttitle=\bfseries,
    enhanced,
  }
}
\newtcolorbox{FailCaseBox}[2][]{failcasebox,title=#2,#1}

\title{Step-wise Rubric Rewards for LLM Reasoning}
\author[1,2,*]{Weichu Xie}
\author[1,*]{Haozhe Zhao}
\author[1,2,*]{Wenpu Liu}
\author[2]{Yongfu Zhu}
\author[1]{Liang Chen}
\author[3]{Minghao Ye}
\author[4]{Zirong Chen}
\author[1,2]{Yuqi Xu}
\author[5,2]{Shuai Dong}
\author[1,2]{Ziyue Wang}
\author[1]{Xinbo Xu}
\author[1]{Kean Shi}
\author[1]{Ruoyu Wu}
\author[2]{Xiaoying Zhang}
\author[2]{Wenqi Shao}
\author[1,\dag]{Baobao Chang}
\author[2]{Nan Duan}
\author[2]{Jiaqi Wang}

\affiliation[1]{Peking University}
\affiliation[2]{JD Explore Academy}
\affiliation[3]{Shanghai Jiao Tong University}
\affiliation[4]{Tsinghua University}
\affiliation[5]{Shanghai Innovation Institute}

\contribution{$^*$Equal contribution, $^\dag$Corresponding author.}

\input{abstract}

\date{\today}

\checkdata[Code]{\url{https://github.com/akarinmoe/SRaR}}
\correspondence{\email{chbb@pku.edu.cn}}
\begin{document}
\maketitle

\input{intro}

\input{conclusion}

\clearpage
\bibliographystyle{plainnat}
\bibliography{main}

\clearpage
\input{appendix}

\end{document}

%% file: common.tex
\usepackage{latexsym}
\usepackage[T1]{fontenc}
\usepackage[utf8]{inputenc}
\usepackage{microtype}
\usepackage{inconsolata}
\usepackage{graphicx}
\usepackage{hyperref}       
\usepackage{url}            
\usepackage{booktabs}       
\usepackage{amsfonts}       
\usepackage{nicefrac}       
\usepackage{stackengine}
\usepackage{microtype}      
\usepackage{colortbl}
\usepackage{xcolor}
\usepackage{amsmath}
\usepackage{amssymb}
\usepackage{amsthm}
\usepackage{mathrsfs}
\usepackage{pifont}
\usepackage{MnSymbol}
\usepackage{balance}
\usepackage{enumitem}
\usepackage{listings}
\usepackage{xcolor}
\usepackage{natbib}
\usepackage{multicol}

\AtBeginDocument{%
  \providecommand\BibTeX{{%
    \normalfont B\kern-0.5em{\scshape i\kern-0.25em b}\kern-0.8em\TeX}}}

\makeatletter
\DeclareRobustCommand\onedot{\futurelet\@let@token\@onedot}
\def\@onedot{\ifx\@let@token.\else.\null\fi}

\usepackage{setspace}
\usepackage{mathtools}

\usepackage{multirow,booktabs}
\usepackage{subcaption}

\newcommand{\owo}[1]{\textsc{OAgents}}

\definecolor{lightgreen}{RGB}{144, 238, 144} 
\definecolor{lightred}{RGB}{255, 105, 97}

\newtcolorbox{promptbox}[2][Prompt]{
colback=black!5!white,
arc=5pt, 
boxrule=0.5pt,
fonttitle=\bfseries,
title=#1, 
before upper={\small}, fontupper=\fontfamily{ptm}\selectfont,
colframe=#2,
}
\definecolor{ogreen}{RGB}{34, 139, 34}

%% file: abstract.tex
\abstract{
Reinforcement Learning with Verifiable Rewards (RLVR) is widely used to improve the reasoning capabilities of large language models, but its reward is derived only from the correctness of the final answer and provides no supervision over intermediate reasoning steps. Recent rubric-based methods, such as Rubrics as Rewards (RaR), introduce finer-grained supervision by scoring rollouts against a structured set of evaluation criteria. However, the resulting rubric scores are still aggregated into a single scalar that is applied to the entire response, leading to three structural weaknesses, namely loss of the multi-criterion rubric structure, uniform supervision of correct and incorrect reasoning steps, and reward hacking in the trained model through unbounded self-correction. On a sample of 1{,}000 problems, we find that 18.2\% of steps within answer-correct responses are themselves wrong yet positively rewarded, while 49.9\% of steps within answer-incorrect responses are in fact correct yet penalized. Therefore, we introduce \textbf{Step-wise Rubrics as Rewards (SRaR)}, an RLVR framework that (i) uses an LLM judge to attribute each rubric item to a specific reasoning step, (ii) normalizes the per-step rubric scores across rollouts so that only steps whose quality varies produce a learning signal, and (iii) combines the resulting per-step reward with the standard outcome reward through a decoupled advantage estimator that keeps the outcome-driven baseline stable. To support training, we further build a 16K-problem rubric dataset by contrastively distilling rubric items from correct and flawed reasoning paths sampled from a strong model and verified against the ground-truth answer. Across six mathematical reasoning benchmarks spanning multiple difficulty levels, SRaR improves the average accuracy over RaR by \textbf{3.57} points on Qwen3-8B-Non-Thinking and by \textbf{2.75} points on Qwen3-32B-Non-Thinking, raises the Faithful Reasoning Rate on AIME~2025 from 34.5\% to 46.7\% (every correct-answer trajectory of SRaR uses entirely correct reasoning steps), and reduces the rate of self-correction looping from 48.1\% to 26.5\%. The gains extend to code generation, where SRaR is best on three of four benchmarks.
}

%% file: intro.tex
\begin{figure}[t]
    \centering
    \includegraphics[width=1\linewidth]{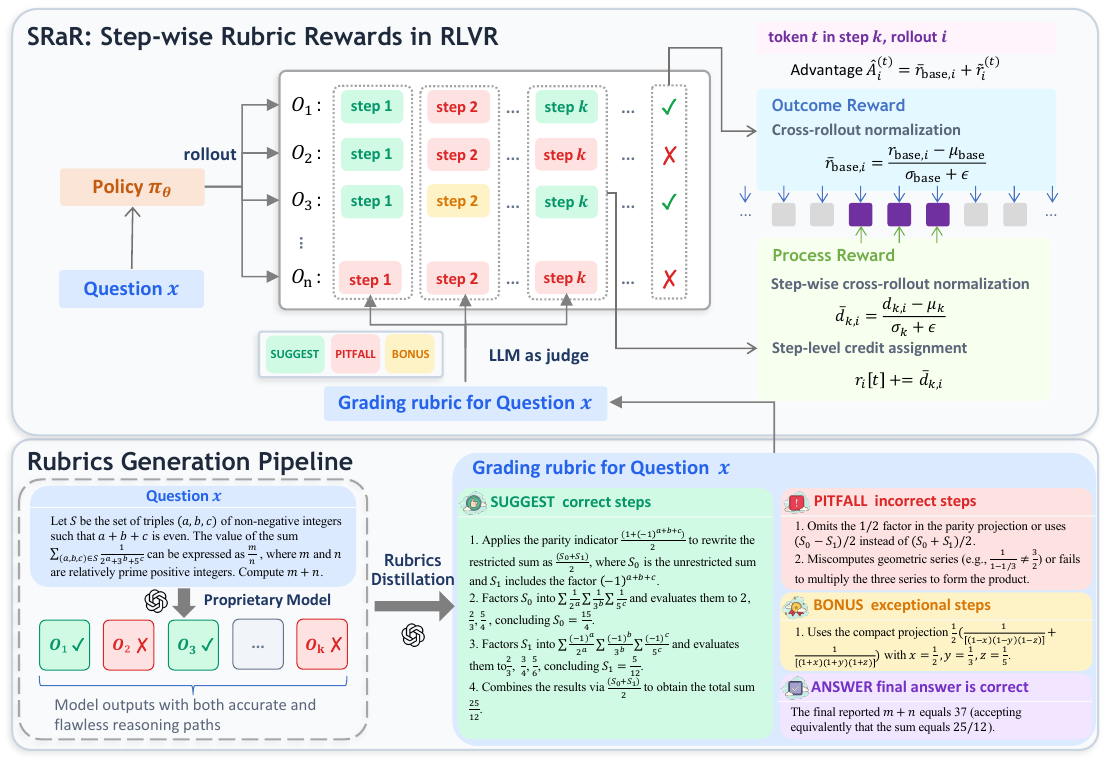}
    \vspace{-1.5em}
    \caption{\textbf{Overview of SRaR.} \textbf{Top:} The SRaR training framework. An LLM judge attributes each rubric item to a specific reasoning step; per-step rewards are normalized across rollouts and combined with outcome rewards via a decoupled advantage estimator. \textbf{Bottom:} The three-stage contrastive rubric distillation pipeline that generates instance-specific rubrics from correct and flawed paths.}
    \label{fig:method_overview}
    \vspace{-0.5em}
  \end{figure}
  
  \section{Introduction}
  
  Large language models (LLMs) have demonstrated remarkable reasoning capabilities across complex tasks such as mathematics, programming, and scientific problem solving~\citep{wei2022chain, openai2024o1, dong2025interleaved, chen2026babyvision, meituanlongcatteam2026longcat, si2026contextskillslanguagemodels, jin2026unveiling, xiao2026rethinking, shen2026chronovision, li2026toward}. To further enhance these capabilities, Reinforcement Learning with Verifiable Rewards (RLVR) has emerged as a widely adopted training paradigm in which rule-based verifiers automatically assess the correctness of model-generated answers and produce reward signals for policy optimization~\citep{deepseek2025r1, shao2024deepseekmath, dong2026rl, liu2026leveraging, wang2026right, shen2026fine}. While effective, the RLVR reward is derived solely from the correctness of the final answer and provides no supervision over the quality of intermediate reasoning steps. To inject finer-grained supervision, recent work turns to rubrics, structured checklists of evaluation criteria that decompose response quality into multiple verifiable aspects and are scored automatically by an LLM judge~\citep{Hashemi_2024, kim2024prometheus}. Methods such as Rubrics as Rewards (RaR)~\citep{yuan2025rar} and RGR-GRPO~\citep{wang2025rgr} use these rubric scores as RL reward signals by aggregating the per-criterion scores into a single scalar reward applied to the entire response (Figure~\ref{fig:motivation}(a)).
  
  Despite providing richer supervision than pure outcome rewards, current rubric-based methods exhibit three structural weaknesses. \textbf{First, existing methods score each rubric item as binary pass/fail and then merge all per-step verdicts into a single response-level scalar, discarding both the type-level and step-level structure of the rubric.} The binary scoring erases the distinction between criteria that reinforce standard reasoning, flag common errors, and credit exceptional insights. The response-level aggregation further discards which step each criterion judges, leaving the policy gradient with no information about the multi-aspect structure that motivated using rubrics in the first place. \textbf{Second, this response-level reward is applied uniformly to every token, so correct and incorrect reasoning steps within the same response receive the same reward signal.} As a result, the policy gradient cannot distinguish steps that actually contributed to the response's quality from steps that should have been corrected. To measure how often this matters in practice, we sample 1{,}000 random solutions from Qwen3-8B-Non-Thinking on the DAPO training set and independently judge each reasoning step. 18.2\% of steps within answer-correct responses are themselves wrong and 49.9\% of steps within answer-incorrect responses are themselves correct (Figure~\ref{fig:motivation}(b)), so any single-scalar response-level reward assigns incorrect supervision to a substantial fraction of intermediate steps. \textbf{Third, the rubric score rewards visible reasoning effort such as self-correction and re-verification, so the model can inflate it without actually improving the answer.} Because existing methods feed this rubric score directly into the policy gradient, the model learns to over-produce such surface behaviors as a form of reward hacking. On AIME~2025, this manifests as severe self-correction looping: 48.1\% of RaR rollouts repeatedly rewrite and re-verify intermediate steps to inflate the rubric reward, often failing to terminate (Section~\ref{sec:looping_analysis}). \textbf{This raises a natural question: how can a reward signal be both interpretable and fine-grained in supervising a model's reasoning?}
  
  \textbf{First, step-attributed rubric judging} preserves the multi-aspect structure that existing aggregation discards. Each rubric item carries one of three types: SUGGEST for standard reasoning, PITFALL for known errors, and BONUS for exceptional insights. An LLM judge evaluates whether each item is satisfied and ties it to the reasoning step it judges, producing per-step, type-differentiated rewards that are retained at the step level rather than merged into a single response-level scalar. \textbf{Second, per-step token assignment with cross-rollout normalization} replaces the uniform token-level reward with step-aligned credit. Each step's reward is broadcast only to its own tokens and normalized across rollouts of the same prompt, so each token receives a learning signal aligned with its own step. \textbf{Third, a decoupled advantage estimator} computes the advantage as a sum of two separately normalized parts: one from the outcome reward, one from the rubric. The outcome part is the standard advantage from accuracy alone; the rubric part is added on top as a bounded, per-step adjustment. Because rubric noise stays in the rubric part, no rollout can inflate its overall advantage by gaming the rubric, removing the incentive behind self-correction looping and yielding more stable training.
  
  All three designs assume access to a problem-specific rubric, yet no public corpus provides rubrics paired with reasoning data at the scale needed for RL training. We therefore introduce a \textbf{contrastive rubric distillation} pipeline that constructs a 16K-problem rubric corpus in three stages: (i) sampling multiple reasoning rollouts from a strong model (GPT-5) per problem; (ii) verifying them against the ground-truth answer to split correct from flawed reasoning paths; and (iii) feeding both groups back to the same model to distill rubric items, each classified as SUGGEST, PITFALL, or BONUS, that capture what separates successful from failed reasoning. By grounding each rubric item in empirically observed failure modes rather than a single reference answer, this procedure yields criteria that target the errors a policy is actually likely to commit during RL training.

  We evaluate \textit{SRaR} on various reasoning benchmarks spanning multiple difficulty levels and model scales. \textit{SRaR} improves average accuracy over the strongest baseline by \textbf{3.57} points on Qwen3-8B and \textbf{2.75} points on Qwen3-32B, with consistent improvements across all benchmarks and model scales. Three targeted analyses on AIME~2025 attribute these gains to the three designs above. Step-attributed judging with per-step token assignment raises the Faithful Reasoning Rate, the fraction of trajectories with both a correct answer and all-correct intermediate steps, from 34.5\% to 46.7\%. The resulting step-level credit assignment reduces self-correction looping, a failure mode in which the model repeatedly regenerates or revisits reasoning steps without converging, from 48.1\% to 26.5\%. Finally, removing cross-rollout step normalization causes training to collapse. The gains also transfer beyond mathematics: on code generation, \textit{SRaR} is best on three of four benchmarks (Appendix~\ref{app:code_domain}). These results demonstrate that step-level rubric supervision improves both final-answer accuracy and intermediate reasoning quality.
  
  Our contributions are summarized as follows:
  \begin{enumerate}[leftmargin=1.5em]
      \item We identify three structural weaknesses of previous methods that use rubric scores as RL reward signals for finer-grained supervision and quantify their impact via a diagnostic study on 1{,}000 random model rollouts.
      \item We propose \textit{SRaR}, an RLVR framework that delivers step-level rubric supervision via three coordinated designs that provide interpretable and fine-grained reward signals for supervising model reasoning.
      \item We release a 16K-problem rubric dataset built via correctness-verified contrastive distillation. \textit{SRaR} outperforms RL baselines by a large margin on six reasoning benchmarks across different model scales, while raising the Faithful Reasoning Rate and substantially reducing the self-correction looping observed in other methods.
  \end{enumerate}
  
  \section{Related Work}
  \begin{figure}[t]
      \centering
      \begin{minipage}[c]{0.49\textwidth}
          \centering
          \includegraphics[width=\linewidth]{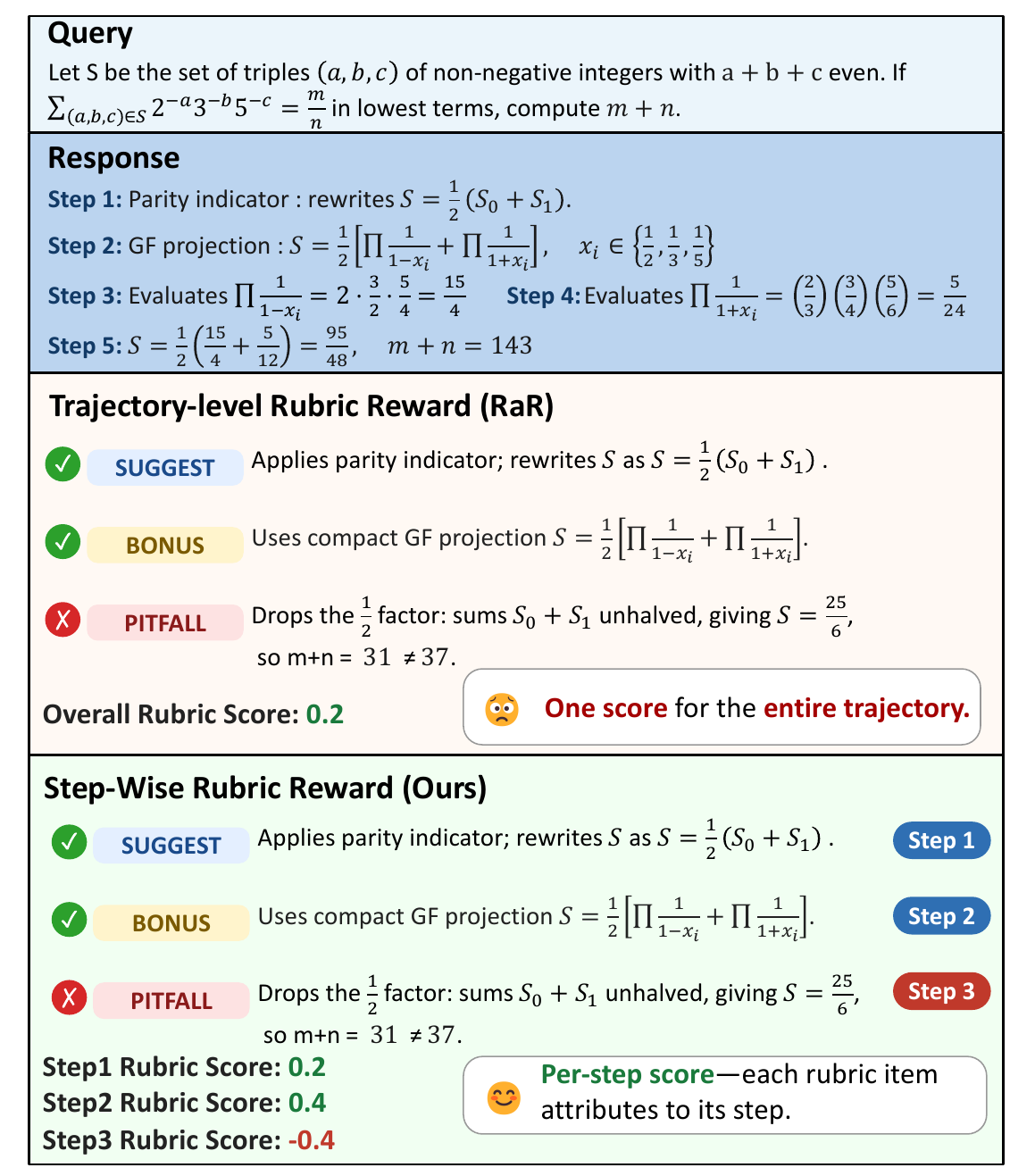}\\[1pt]
          {\small (a) Response-level (RaR) vs.\ step-wise (SRaR) rubric reward}
      \end{minipage}%
      \hspace{0.01\textwidth}%
      \begin{minipage}[c]{0.49\textwidth}
          \centering
          \vspace{-0.5em}
          \resizebox{\linewidth}{!}{%
          \renewcommand{\arraystretch}{2.15}
          \setlength{\tabcolsep}{3pt}
          \begin{tabular}{@{} l c c c l @{}}
          \toprule
          \rowcolor{headblue}
          \textbf{Category} & \textbf{\#\,Samples} & \textbf{\#\,Steps} & \textbf{Reward} & \textbf{Indiscriminate} \\
          \midrule
          \rowcolor{goodgreen}
          A$\checkmark$S$\checkmark$ & 194 & 1012 & $+$ & \textemdash \\
          \rowcolor{badred}
          A$\checkmark$S$\times$ & 194 & 916 & $+$ & 350 wrong \\
          \rowcolor{badred}
          A$\times$S$\checkmark$ & 18 & 101 & $-$ & 101 correct \\
          \rowcolor{badred}
          A$\times$S$\times$ & 591 & 3151 & $-$ & 1521 correct \\[1pt]
          \midrule
          \multicolumn{5}{@{}l}{\footnotesize \textbf{Indiscriminate reward under response-level signal:}} \\
          \multicolumn{5}{@{}l}{\footnotesize Wrong steps receiving same\,$+$\,as correct (A$\checkmark$):} \\
          \multicolumn{5}{@{}l}{\footnotesize \quad $\dfrac{350}{1012+916}=\dfrac{350}{1928}=\mathbf{18.2\%}$} \\
          \multicolumn{5}{@{}l}{\footnotesize Correct steps receiving same\,$-$\,as wrong (A$\times$):} \\[1pt]
          \multicolumn{5}{@{}l}{\footnotesize \quad $\dfrac{101+1521}{101+3151}=\dfrac{1622}{3252}=\mathbf{49.9\%}$} \\[4pt]
          \bottomrule
          \end{tabular}%
          }\\[0.8em]
          {\small (b) Indiscriminate reward under response-level signal}
      \end{minipage}
      \caption{\textbf{Motivation for step-wise rubric reward.} \textbf{(a)} RaR applies a single scalar rubric reward to the entire response, so every step receives the same signal; SRaR decomposes the rubric into per-step rewards aligned with each reasoning step. \textbf{(b)} Diagnostic study on 1{,}000 random Qwen3-8B-Non-Thinking solutions from the DAPO training set, each parsed into reasoning steps and independently judged (details in Section~\ref{sec:process_outcome_alignment}). A$\checkmark$/A$\times$ denotes correct/incorrect final answer; S$\checkmark$/S$\times$ denotes all steps correct / at least one step wrong. Under a single-scalar response-level reward, \textbf{18.2\%} of steps within answer-correct responses are wrong yet receive the same $+$ reward as correct steps, and \textbf{49.9\%} of steps within answer-incorrect responses are correct yet receive the same $-$ penalty as wrong steps.}
      \label{fig:motivation}
  \end{figure}
  
  \subsection{Process Supervision for Reasoning}
  
  RLVR with GRPO and DAPO~\citep{yu2025dapo} provides only sparse binary outcome signals with no guidance on intermediate steps. Process Reward Models (PRMs)~\citep{lightman2023lets, wang2024math, luo2024improve} score individual steps but require massive human annotations and produce opaque scalars with no explanation for why a step receives its score. GRPO-VPS~\citep{zhang2025grpo_vps} avoids auxiliary models by measuring the policy's own conditional probability of the correct answer at each step boundary, but the resulting signal is still a single opaque scalar per step that offers no interpretable diagnosis of reasoning quality.
  
  \subsection{Rubric-Based Evaluation and Rewards}
  
  Rubric-based evaluation decomposes quality into interpretable criteria without auxiliary models. \citet{Hashemi_2024} introduce LLM-Rubric for calibrated multi-aspect evaluation; \citet{kim2024prometheus} train open-source rubric-guided evaluators; RaR~\citep{yuan2025rar} and RGR-GRPO~\citep{wang2025rgr} use rubric scores directly as RL rewards. However, all existing rubric-based RL methods aggregate scores into a single trajectory-level scalar, forfeiting per-step credit: correct steps in failed trajectories are penalized and incorrect steps in successful ones are reinforced. SRaR addresses this overlooked problem by attributing each rubric item to its corresponding reasoning step and normalizing rewards at the step level.
  
  \section{Method}
  
  \subsection{Preliminary}
  \textbf{RLVR for LLM Reasoning.}
  Reinforcement Learning with Verifiable Rewards (RLVR) trains a policy $\pi_\theta$ by rewarding responses whose final answer matches the ground truth~\citep{deepseek2025r1, shao2024deepseekmath}. We adopt Group Relative Policy Optimization (GRPO)~\citep{shao2024deepseekmath}: for each prompt $x \in \mathcal{X}$, where $\mathcal{X}$ denotes the space of problem prompts, GRPO samples $n$ rollouts $\{y_i\}_{i=1}^n \sim \pi_\theta(\cdot \mid x)$, scores each with a scalar reward $r_i$, and computes the advantage by group normalization
  \begin{equation}
      A_i = \frac{r_i - \mu}{\sigma + \epsilon}, \quad \mu = \frac{1}{n}\sum_{j=1}^{n} r_j, \quad \sigma^2 = \frac{1}{n}\sum_{j=1}^{n}(r_j - \mu)^2.
  \end{equation}
  The same scalar $A_i$ is broadcast as the per-token advantage at every position $t$ of $y_i$.
  
  \textbf{Problem Formulation.}
  Each problem $x \in \mathcal{X}$ in our training set $\mathcal{D} = \{(x, a^*, \mathcal{R}_x)\}$ carries a ground-truth answer $a^*$ and a rubric $\mathcal{R}_x = \{(c_j, \tau_j)\}_{j=1}^{M}$ with criterion text $c_j$ and type $\tau_j \in \{\texttt{SUGGEST}, \texttt{PITFALL}, \texttt{BONUS}\}$. Given a response $y = (y_1, \ldots, y_T) \in \mathcal{V}^T$, a token sequence over the vocabulary $\mathcal{V}$, organized into $K$ steps with token spans $\{(a_k, b_k)\}_{k=1}^{K}$, our task is to construct a token-level advantage $\hat{A}_i^{(t)}$ for each rollout $i$ and position $t$, jointly informed by outcome correctness and rubric verdicts, as a drop-in replacement for $A_i$. A summary of notation is provided in Table~\ref{tab:notation} (Appendix~\ref{app:notation}).

  \subsection{Rubrics Design}
  
  A rubric $\mathcal{R}_x$ comprises three item types, each contributing a distinct sign to the policy gradient (Section~\ref{subsec:srar}).
  \textbf{SUGGEST} items reinforce standard problem-solving paradigms (clear step structure, logical coherence, explicit derivations), contributing positive deltas.
  \textbf{PITFALL} items penalize known failure modes (arithmetic errors, logical leaps, hallucinations), contributing negative deltas.
  \textbf{BONUS} items reward valuable optional behaviors (alternative solution paths, cross-validation), contributing additional positive deltas.
  This three-aspect schema separates baseline correctness, error avoidance, and exceptional quality, each bounded by a per-type budget.
  
  \subsection{Step-wise Rubrics Rewards in RLVR (SRaR)}
  \label{subsec:srar}
  
  Figure~\ref{fig:method_overview} (top) provides an overview. SRaR realizes three coordinated designs: \textbf{step-attributed rubric judging}, \textbf{per-step token assignment with cross-rollout normalization}, and a \textbf{decoupled advantage estimator}.
  
  \textbf{Step-attributed rubric judging.}
  An LLM judge ties each rubric item to the reasoning step it judges, producing a per-step, per-criterion verdict instead of a single response-level scalar. To enable step attribution, the policy is prompted to emit responses in a structured step format (``\texttt{\#\#\# Step N:} ...'') with the final answer in \texttt{\textbackslash boxed\{\}}, and step token spans $\{(a_k, b_k)\}$ are recovered via the tokenizer's offset mapping. For each response $y$, the judge $\mathcal{J}$~\citep{zheng2024judging} returns
  \begin{equation}
      \mathcal{J}(y, \mathcal{R}_x) = \{(s_j, k_j)\}_{j=1}^{M}, \quad s_j \in \{0, 1\}, \; k_j \in \{1, \ldots, K\},
  \end{equation}
  where $s_j$ indicates whether criterion $j$ is satisfied and $k_j$ records the step it judges.
  
  \textbf{Per-step token assignment with cross-rollout normalization.}
  Each judged rubric item is converted into a per-step reward, normalized across rollouts at the step level, and broadcast only to the tokens of the step it judges. This replaces the uniform token-level reward of standard GRPO with step-aligned credit.
  
  We first map each verdict to a signed delta whose magnitude depends on its rubric type:
  \begin{equation}
      \delta_j = \begin{cases}
          +R_{\texttt{SUG}} / N_{\texttt{SUG}} & \tau_j = \texttt{SUGGEST}, \; s_j = 1 \\[2pt]
          -|R_{\texttt{PIT}}| / N_{\texttt{PIT}} & \tau_j = \texttt{PITFALL}, \; s_j = 1 \\[2pt]
          +R_{\texttt{BON}} / N_{\texttt{BON}} & \tau_j = \texttt{BONUS}, \; s_j = 1 \\[2pt]
          0 & \text{otherwise,}
      \end{cases}
      \label{eq:delta}
  \end{equation}
  where $R_{\tau} \in \mathbb{R}$ is the signed reward budget for type $\tau \in \{\texttt{SUGGEST}, \texttt{PITFALL}, \texttt{BONUS}\}$, with $R_{\texttt{SUG}}, R_{\texttt{BON}} > 0$ and $R_{\texttt{PIT}} < 0$, and $N_{\tau}$ counts the items of that type in the current sample, keeping each category's contribution bounded regardless of rubric list length. Deltas attributed to the same step $k$ are summed into a per-step raw rubric reward $d_{k,i} = \sum_{j: k_j = k} \delta_j$ for rollout $i$, representing the total rubric credit that step $k$ receives before normalization.
  
  To extract a learning signal that reflects relative step quality across rollouts, we normalize $d_{k,i}$ per step across rollouts of the same prompt:
  \begin{equation}
      \bar{d}_{k,i} = \frac{d_{k,i} - \mu_k}{\sigma_k + \epsilon}, \quad \mu_k = \frac{1}{|\mathcal{G}_k|}\sum_{i \in \mathcal{G}_k} d_{k,i},
      \label{eq:step_norm}
  \end{equation}
  where $\mathcal{G}_k$ is the set of rollouts whose judge evaluation contains at least one rubric item attributed to step $k$, and $\bar{d}_{k,i} = 0$ when $|\mathcal{G}_k| \leq 1$ since a single sample provides no relative signal. This mirrors GRPO's group normalization, applied independently per step: steps that are consistently good or bad across rollouts receive near-zero advantage, while steps whose quality varies receive a strong learning signal. The normalized $\bar{d}_{k,i}$ is then broadcast to every token within the step's span:
  \begin{equation}
      \tilde{r}_i^{(t)} = \bar{d}_{k,i} \quad \text{for all } t \in [a_k, b_k],
  \end{equation}
  so that all tokens of a reasoning step share the same rubric-derived signal.
  
  \textbf{Decoupled advantage estimator.}
  A naive substitution of the combined base-and-rubric reward into the GRPO advantage would let rubric noise enter the group baseline, allowing a rollout to inflate its overall advantage by gaming the rubric. We instead compute the advantage as a sum of two separately normalized parts:
  \begin{equation}
      \hat{A}_i^{(t)} = \underbrace{\frac{r_{\text{base},i} - \mu_{\text{base}}}{\sigma_{\text{base}} + \epsilon}}_{\text{outcome advantage}} + \underbrace{\tilde{r}_i^{(t)}}_{\text{rubric offset}},
  \end{equation}
  where $r_{\text{base},i} = (1-\lambda)\,r_{\text{acc},i} + \lambda\,r_{\text{fmt},i}$ combines an accuracy reward $r_{\text{acc},i} \in \{0,1\}$ for whether the extracted answer is correct and a format reward $r_{\text{fmt},i} \in \{0,1\}$ for whether the response contains both step headers and a \texttt{\textbackslash boxed\{\}} answer, and $\mu_{\text{base}}, \sigma_{\text{base}}$ are computed across rollouts of the same prompt. The outcome term is the standard GRPO advantage from accuracy alone; the rubric term is the bounded, per-step offset constructed above. Because rubric noise stays in the offset and never enters the group baseline, no rollout can inflate its overall advantage by gaming the rubric, removing the incentive behind self-correction looping and yielding more stable training. We provide formal analyses of the estimator's gradient decomposition, variance, and unbiasedness properties in Appendix~\ref{app:theory}. The full procedure is summarized in Algorithm~\ref{alg:srar}.
  
  \begin{algorithm}[t]
  \caption{SRaR: Step-wise Rubric Rewards in RLVR}
  \label{alg:srar}
  \begin{algorithmic}[1]
  \Require Policy $\pi_\theta$, training set $\mathcal{D}$ with rubrics, LLM judge $\mathcal{J}$, group size $n$
  \For{each training iteration}
      \State Sample prompts $\{x\}$ from $\mathcal{D}$; for each $x$, sample $n$ rollouts $\{y_i\}_{i=1}^n \sim \pi_\theta(\cdot \mid x)$
      \For{each prompt $x$}
          \For{$i = 1, \ldots, n$}
              \State $r_{\text{base},i} \gets (1 - \lambda)\, r_{\text{acc},i} + \lambda\, r_{\text{fmt},i}$
              \State Detect step spans $\{(a_k^i, b_k^i)\}_{k=1}^{K_i}$ via tokenizer offset mapping
              \State $\{(s_j, k_j)\}_j \gets \mathcal{J}(y_i, \mathcal{R}_x)$ \Comment{Step-attributed rubric judging}
              \State $d_{k,i} \gets \sum_{j:\, k_j = k} \delta_j$ for each step $k$
          \EndFor
          \For{each step $k$ present in any rollout}
              \State $\mathcal{G}_k \gets \{i : d_{k,i} \text{ exists}\}$; $\bar{d}_{k,i} \gets 0$ if $|\mathcal{G}_k| \leq 1$ else $(d_{k,i} - \mu_k) / (\sigma_k + \epsilon)$
          \EndFor
          \For{$i = 1, \ldots, n$}
              \State $\tilde{r}_i^{(t)} \gets \bar{d}_{k,i}$ for all $t \in [a_k^i, b_k^i]$, each step $k$ \Comment{Per-step token assignment}
              \State $A_{\text{base},i} \gets (r_{\text{base},i} - \mu_{\text{base}}) / (\sigma_{\text{base}} + \epsilon)$
              \State $\hat{A}_i^{(t)} \gets A_{\text{base},i} + \tilde{r}_i^{(t)}$ for all $t$ \Comment{Decoupled advantage}
          \EndFor
      \EndFor
      \State Update $\theta$ via clipped PPO with token-level advantages $\{\hat{A}_i^{(t)}\}$
  \EndFor
  \end{algorithmic}
  \end{algorithm}


  
  
  
  \subsection{Contrastive Rubric Distillation}
  \label{sec:data_pipeline}
  \label{subsec:rubric_pipeline}
  
  No public corpus provides problem-specific rubrics paired with reasoning data at scale. We construct a 16K-problem rubric corpus via contrastive distillation (Figure~\ref{fig:method_overview}, bottom) in three stages.
  \textbf{(1) Reasoning rollout generation.} We prompt GPT-5 to sample multiple step-by-step trajectories per problem with final answers in a standardized tag.
  \textbf{(2) Trajectory verification.} Each trajectory is matched against the ground-truth answer, partitioning rollouts into correct and flawed paths.
  \textbf{(3) Contrastive rubric distillation.} Both correct and flawed paths are fed back to the model, which contrasts structural differences to distill a fine-grained rubric list. Grounding items in empirically observed failures yields criteria targeting errors a policy is likely to commit during training.
  
  Prompts are in Appendix~\ref{app:prompt_rollout} and~\ref{app:prompt_rubric}; rubric examples in Appendix~\ref{app:rubric_examples}; quality evaluation in Appendix~\ref{app:rubric_quality}; judge validation in Appendix~\ref{app:judge_agreement}.

  \section{Experiment}
  
  \subsection{Experimental Setup}
  
  \textbf{Training Details.}
  We build the training set by applying our pipeline to the DAPO math dataset (statistics in Appendix~\ref{app:dataset-stats}) and adopt the same hyperparameters and training setting as DAPO except for the reward signal.
  Policy models are Qwen3-8B-Non-Thinking and Qwen3-32B-Non-Thinking~\citep{yang2025qwen3}, trained on $8\times$ NVIDIA H200 GPUs with rubric budgets $R_{\texttt{SUG}} = 0.8$, $R_{\texttt{PIT}} = -1.0$, $R_{\texttt{BON}} = 1.0$ and format weight $\lambda = 0.1$.
  Details are in Appendix~\ref{app:hparams} and~\ref{app:prompt_format}.
  
  \textbf{Evaluation Configuration.}
  We evaluate on six benchmarks: AIME~2024/2025, AMC~2023, MATH500~\citep{hendrycks2021measuring}, Minerva Math~\citep{lewkowycz2022solving}, and Olympiad Bench~\citep{he2024olympiadbench}.
  Decoding uses temperature 0.7, top-$p$ 0.8, presence penalty 1.5, repetition penalty 1.0, max length 10{,}240 tokens.
  AIME~2024/2025 and AMC~2023 report Avg@32~\citep{wang2023selfconsistency}; remaining benchmarks use a single generation.
  
  \textbf{Baselines.}
  We compare against GRPO~\citep{shao2024deepseekmath}, DAPO~\citep{yu2025dapo}, RaR~\citep{yuan2025rar}, GRPO-VPS~\citep{zhang2025grpo_vps}, and RGR-GRPO~\citep{wang2025rgr}. The latter two are re-implemented under the same training recipe for a fair comparison.
  
  \textbf{Rubric LLM Judge.}
  We use GPT-OSS-20B as the rubric judge (prompt in Appendix~\ref{app:prompt_judge}). Robustness to judge choice is validated with Qwen3-30B-A3B-Instruct as an alternative (Appendix~\ref{app:judge_robustness}).
  
  \subsection{Experimental Results}
  
  \subsubsection{Main Results}
  
  \begin{table}[t]
  \centering
  \caption{Main results on six mathematical reasoning benchmarks for Qwen3-8B-Non-Thinking and Qwen3-32B-Non-Thinking. $\Delta$ denotes the improvement of SRaR over the strongest existing baseline (RaR). The best results are \textbf{bolded}, second best are \underline{underlined}, and our method is highlighted in \colorbox{lightred}{light red}. For AIME~2024/2025 and AMC~2023, we report Avg@32.}
  \label{tab:main_results_8b}
  \vspace{0em}
  \resizebox{0.95\textwidth}{!}{%
  \begin{tabular}{l ccc cccc}
  \toprule
  \textbf{Method} & \textbf{AIME24} & \textbf{AIME25} & \textbf{AMC23} & \textbf{MATH500} & \textbf{Minerva} & \textbf{Olympiad} & \textbf{Avg.} \\
  \midrule
  \rowcolor{gray!15}
  \multicolumn{8}{l}{\textit{Qwen3-8B-Non-Thinking}} \\
  Baseline & 26.88 & 23.33 & 72.03 & 86.8 & 55.15 & 56.23 & 53.40 \\
  + GRPO & 36.56 & 27.71 & 79.06 & 89.2 & 55.51 & 63.20 & 58.54 \\
  + GRPO-VPS & 44.58 & 32.40 & 83.20 & 88.60 & \underline{61.03} & 60.24 & 61.68 \\
  + RGR-GRPO & 45.31 & 34.69 & 81.09 & 90.40 & \underline{61.03} & 66.91 & 63.24 \\
  + DAPO & 52.60 & 41.35 & \underline{90.23} & 92.0 & 59.56 & \underline{67.80} & 67.26 \\
  + RaR & \underline{57.50} & \underline{43.23} & 88.20 & \underline{92.6} & \textbf{64.71} & 67.51 & \underline{68.96} \\
  \rowcolor{lightred}
  + \textbf{SRaR (Ours)} & \textbf{61.15} & \textbf{49.06} & \textbf{92.73} & \textbf{94.4} & \textbf{64.71} & \textbf{73.15} & \textbf{72.53} \\
  $\Delta$ & \textcolor{green!50!black}{+3.65} & \textcolor{green!50!black}{+5.83} & \textcolor{green!50!black}{+4.53} & \textcolor{green!50!black}{+1.80} & \textcolor{green!50!black}{+0.00} & \textcolor{green!50!black}{+5.64} & \textcolor{green!50!black}{+3.57} \\
  \midrule
  \rowcolor{gray!15}
  \multicolumn{8}{l}{\textit{Qwen3-32B-Non-Thinking}} \\
  Baseline & 27.19 & 21.67 & 71.48 & 85.8 & 62.87 & 58.75 & 54.63 \\
  + GRPO & 37.92 & 30.21 & 80.78 & 91.0 & 63.97 & 66.02 & 61.65 \\
  + GRPO-VPS & 43.02 & 26.56 & 80.08 & 90.20 & 63.97 & 64.54 & 61.40 \\
  + RGR-GRPO & 47.71 & 36.46 & 87.42 & 94.40 & \underline{68.01} & \underline{72.40} & 67.73 \\
  + DAPO & 61.25 & 47.81 & 93.67 & \textbf{95.4} & 63.60 & 72.11 & 72.31 \\
  + RaR & \underline{63.65} & \underline{50.94} & \underline{94.22} & \underline{94.8} & 65.44 & 69.29 & \underline{73.06} \\
  \rowcolor{lightred}
  + \textbf{SRaR (Ours)} & \textbf{64.17} & \textbf{51.77} & \textbf{94.77} & \textbf{95.4} & \textbf{72.06} & \textbf{76.71} & \textbf{75.81} \\
  $\Delta$ & \textcolor{green!50!black}{+0.52} & \textcolor{green!50!black}{+0.83} & \textcolor{green!50!black}{+0.55} & \textcolor{green!50!black}{+0.60} & \textcolor{green!50!black}{+6.62} & \textcolor{green!50!black}{+7.42} & \textcolor{green!50!black}{+2.75} \\
  \bottomrule
  \end{tabular}%
  }
  \vspace{-0.5em}
  \end{table}
  
  Table~\ref{tab:main_results_8b} reports the main results. SRaR achieves the highest average accuracy at both scales: 72.53 on Qwen3-8B (+3.57 over RaR) and 75.81 on Qwen3-32B (+2.75 over RaR), and improves over GRPO-VPS by 10.85/14.41 points and over RGR-GRPO by 9.29/8.08 points. The advantage over GRPO-VPS stems from rubric items describing specific reasoning behaviors rather than relying on probability increments of a reference answer, enabling credit attribution to step content rather than statistical effects. The advantage over RaR and RGR-GRPO comes from routing each rubric item to the step it evaluates, eliminating indiscriminate supervision where correct steps in incorrect trajectories are penalized. We validate this in Section~\ref{sec:process_outcome_alignment}.
  
  The gains concentrate on longer benchmarks: at 8B, AIME~2025 (+5.83), Olympiad Bench (+5.64), and AMC~2023 (+4.53) improve most, while MATH500 (+1.80) and Minerva move less. Across scales, SRaR raises the base model's average by 19.13 points at 8B and 21.18 points at 32B. Training dynamics are in Appendix~\ref{app:training_dynamics}.
  
  \subsubsection{Think Longer and Better}
  \label{sec:think_long_better}
  
  \begin{figure}[t]
      \centering
      \begin{minipage}[t]{0.48\textwidth}
          \centering
          \vspace{0pt}
          \includegraphics[width=\textwidth]{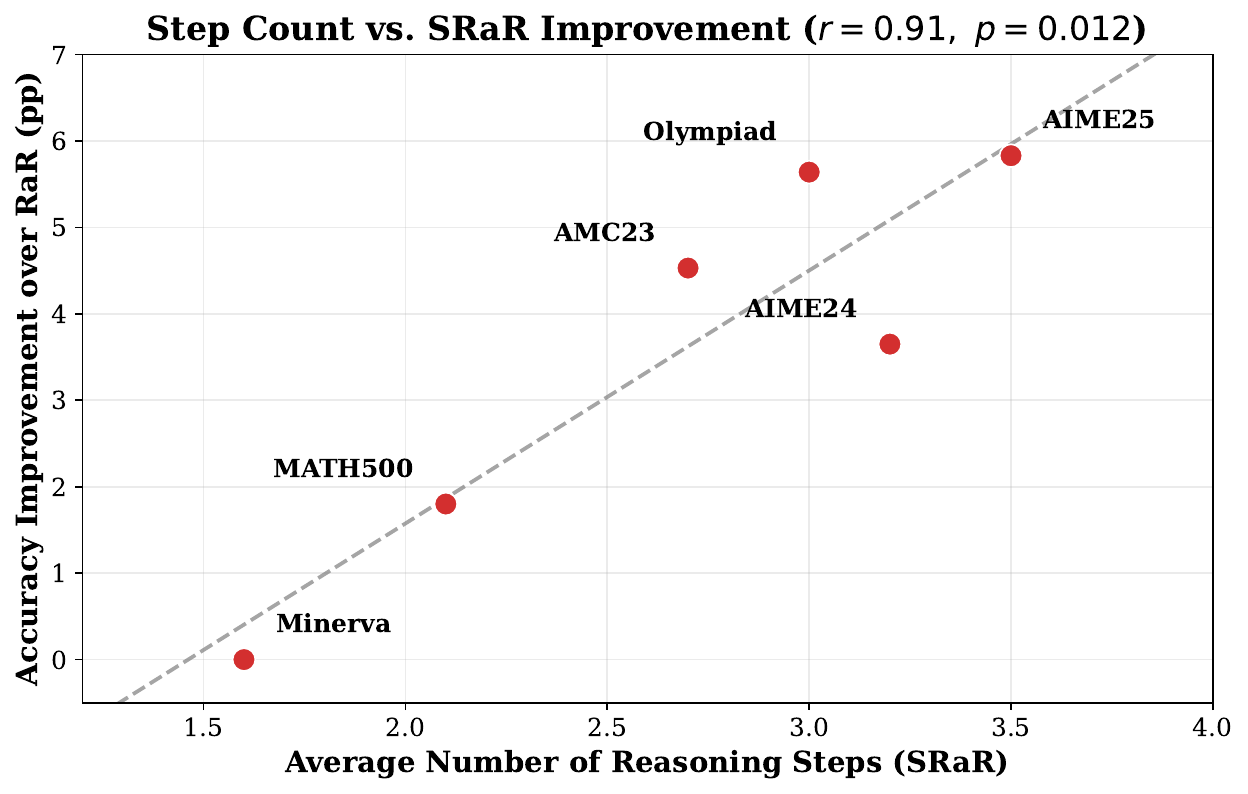}
      \end{minipage}
      \hfill
      \begin{minipage}[t]{0.5\textwidth}
          \centering
          \vspace{0pt}
          \includegraphics[width=\textwidth]{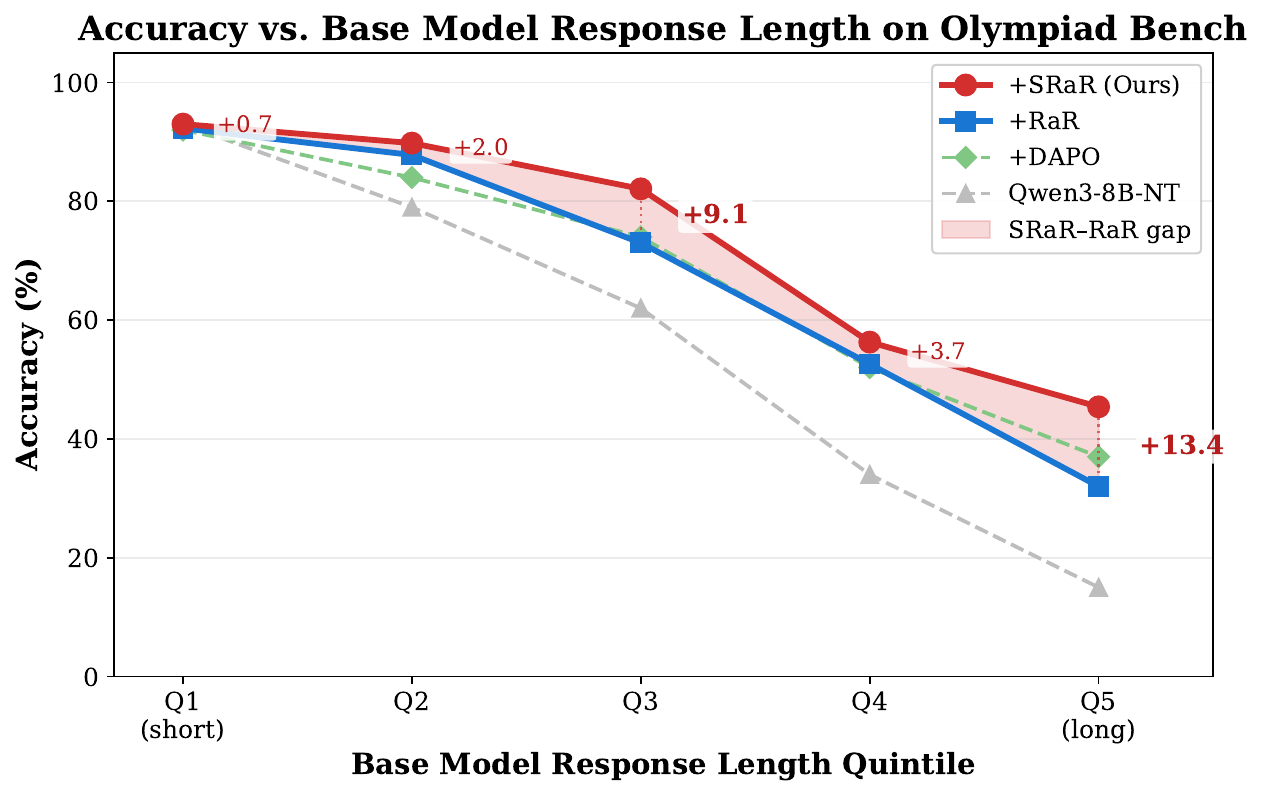}
      \end{minipage}
      \caption{\textbf{(Left)} Correlation between the average number of reasoning steps per benchmark and SRaR's accuracy improvement over RaR (Pearson $r = 0.91$, $p = 0.012$). Benchmarks requiring more reasoning steps benefit more from step-wise credit assignment. \textbf{(Right)} Accuracy vs.\ problem difficulty on Olympiad Bench (674 problems), where difficulty quintiles are defined by the base model's response length. SRaR outperforms all baselines across all five quintiles, with the largest margin on the hardest problems (Q5: +13.4 pp over RaR).}
      \label{fig:think_longer_better}
  \end{figure}
  
  \textbf{Why gains vary across benchmarks.}
  We correlate the average number of reasoning steps per benchmark with SRaR's improvement over RaR (Figure~\ref{fig:think_longer_better}, left). The correlation is strong (Pearson $r = 0.91$, $p = 0.012$): benchmarks requiring more steps (AIME~2025: 3.6, Olympiad: 3.0) benefit substantially more than those with shorter solutions (Minerva: 1.6).
  
  \textbf{Harder problems benefit more.}
  We stratify Olympiad Bench problems into difficulty quintiles by base model response length (Figure~\ref{fig:think_longer_better}, right). SRaR outperforms RaR in all quintiles, with the gap widening from +0.7 pp (Q1) to +13.4 pp (Q5). On Q5, SRaR also generates shorter responses (3{,}645 vs.\ 5{,}933 tokens), confirming gains from effective reasoning rather than exhaustive self-revision.
  
  \subsection{Analysis}
  
  \subsubsection{Process-Outcome Alignment of Step-wise Rewards}
  \label{sec:process_outcome_alignment}
  
  To examine whether step-level reward assignment improves intermediate reasoning quality beyond final-answer accuracy, we conduct a diagnostic on AIME~2025 (single-pass). For each model, we parse solutions into steps and an independent LLM judge evaluates each step's correctness conditioned on preceding steps (prompt in Appendix~\ref{app:prompt_step_judge}).
  
  As shown in Table~\ref{tab:process_outcome}, the base model reaches only 60.8\% step accuracy, confirming that outcome-only supervision fails to prevent error accumulation. RaR raises step accuracy to 80.8\% but still has 2 correct-answer trajectories with erroneous steps (A$\checkmark$S$\times{=}2$) and 4 incorrect-answer trajectories with all-correct steps (A$\times$S$\checkmark{=}4$). SRaR achieves the highest step accuracy (85.2\%) and answer accuracy (46.7\%), and is the only model with A$\checkmark$S$\times{=}0$: every correct-answer trajectory uses entirely correct reasoning. SRaR also attains the highest Faithful Reasoning Rate (46.7\% vs.\ 34.5\%) and the lowest Misaligned Trajectory Rate (10.0\% vs.\ 20.7\%), showing that step-level rewards align intermediate quality with outcome correctness. One RaR trajectory is excluded because it does not contain a parseable reasoning trace, giving effective totals of 30, 29, and 30 trajectories for Qwen3-8B, RaR, and SRaR respectively. Case studies are in Appendix~\ref{app:case_process_outcome}.
  
  \subsubsection{Reduction of Repetitive Looping}
  \label{sec:looping_analysis}
  
  A common failure mode of RLVR-trained models is repetitive looping: unbounded self-correction cycles that revisit computations without converging (case study in Appendix~\ref{app:case_looping}).
  We quantify this on 960 AIME~2025 samples (30 problems, 32 rollouts each), classifying a response as looping if any of four criteria holds: (i) $>20$ self-correction phrases; (ii) \texttt{\#\#\# Step 1} appearing more than once; (iii) repeated step headings; or (iv) $>10\%$ duplicate paragraphs.
  
  \begin{table}[t]
  \centering
  \caption{Step-level reasoning quality analysis on AIME~2025. \textbf{Step Acc.} is the percentage of steps judged correct. The four combination columns show counts per (answer $\times$ step correctness) category. \textbf{FRR} (Faithful Reasoning Rate) is A$\checkmark$S$\checkmark$\,/\,total ($\uparrow$), the fraction of trajectories with both a correct answer and all-correct steps. \textbf{MTR} (Misaligned Trajectory Rate) is (A$\checkmark$S$\times$ + A$\times$S$\checkmark$)\,/\,total ($\downarrow$), the fraction of trajectories whose step-level judgement disagrees with the outcome, and serves as the population-level analogue of the misalignment in Figure~\ref{fig:motivation}(b).}
  \label{tab:process_outcome}
  \vspace{0em}
  \small
  \setlength{\tabcolsep}{4pt}
  \begin{tabular}{l c c c c c c c c}
  \toprule
  \textbf{Model} & \textbf{Ans.} & \textbf{Step} & \multicolumn{4}{c}{\textbf{Answer $\times$ Step}} & \textbf{FRR} & \textbf{MTR} \\
  \cmidrule(lr){4-7}
  & \textbf{Acc.} & \textbf{Acc.} & \textbf{A$\checkmark$S$\checkmark$} & \textbf{A$\checkmark$S$\times$} & \textbf{A$\times$S$\checkmark$} & \textbf{A$\times$S$\times$} & \textbf{(\%)\,$\uparrow$} & \textbf{(\%)\,$\downarrow$} \\
  \midrule
  Qwen3-8B & 13.3 & 60.8 & 4 & \textbf{0} & 4 & 22 & 13.3 & \underline{13.3} \\
  + RaR  & \underline{41.4} & \underline{80.8} & \underline{10} & 2 & \underline{4} & \underline{13} & \underline{34.5} & 20.7 \\
  + \textbf{SRaR (Ours)} & \textbf{46.7} & \textbf{85.2} & \textbf{14} & \textbf{0} & \textbf{3} & \textbf{13} & \textbf{46.7} & \textbf{10.0} \\
  \bottomrule
  \end{tabular}
  \vspace{-1em}
  \end{table}
  
  \begin{table}[t]
  \centering
  \caption{Repetitive looping analysis on AIME~2025. \textbf{Loop Rate} is the percentage of samples exhibiting any form of looping. Accuracy is reported for looping vs.\ non-looping responses. \textbf{Mean S-C} denotes the average count of self-correction phrases per response.}
  \label{tab:looping}
  \vspace{0em}
  \small
  \begin{tabular}{l c c c c}
  \toprule
  \textbf{Model} & \textbf{Loop Rate (\%)} & \textbf{Acc. (loop)} & \textbf{Acc. (no loop)} & \textbf{Mean S-C} \\
  \midrule
  Qwen3-8B & 56.8 & 25.9 & 16.1 & \textbf{0.9} \\
  + RaR & 48.1 & \underline{24.2} & \textbf{53.4} & 25.4 \\
  + \textbf{SRaR (Ours)} & \textbf{26.5} & \textbf{42.5} & \underline{52.8} & \underline{14.8} \\
  \bottomrule
  \end{tabular}
  \vspace{-0.5em}
  \end{table}
  
  As shown in Table~\ref{tab:looping}, RaR's looping rate is 48.1\% with a mean of 25.4 self-correction phrases, and its looping samples achieve only 24.2\% accuracy vs.\ 53.4\% for non-looping ones. SRaR reduces the looping rate to 26.5\% with mean self-correction count 14.8, and its looping samples reach 42.5\% accuracy, suggesting productive reflection rather than uncontrolled repetition.
  
  We attribute this to step-level credit assignment: each step receives an independent reward, so the model can commit to completed steps rather than perpetually revisiting them.
  
  \subsection{Ablation Study}
  \textbf{Rubric aspect ablation.}
  We test whether PITFALL and BONUS provide value beyond Suggest-only. As shown in Table~\ref{tab:rubric_ablation}, SRaR outperforms RaR under every single-aspect configuration and under the full rubric set (+2.86 SUGGEST-only, +4.28 PITFALL-only, +2.89 BONUS-only, +3.57 multi-aspect), confirming the primary gain comes from step-wise credit assignment rather than any one rubric aspect. Each aspect also supervises effectively on its own (71.39, 71.97, and 70.89 respectively), and combining all three raises accuracy to 72.53, indicating complementary supervision beyond positive guidance alone.
  
  \textbf{Method design ablation.}
  We ablate three components on Qwen3-8B (Table~\ref{tab:component_ablation} in Appendix~\ref{app:component_ablation}). Removing the step-structured format prompt yields 67.85 average accuracy, matching standard DAPO (67.26), since step boundaries are required for attribution. Removing cross-rollout step normalization causes collapse to 56.76: without normalization, group-invariant rubric signals dominate the gradient, and the model degenerates to a single step (Appendix~\ref{app:case_collapse}). Replacing the decoupled advantage with a merged RaR-style scalar yields 72.06, slightly below full SRaR (72.53), because merging injects rubric variance into the GRPO baseline statistics.

%% file: conclusion.tex
\section{Conclusion}

We present \textit{Step-wise Rubrics as Rewards (SRaR)}, an RLVR framework that delivers fine-grained supervision during RL training via step-level rubrics. Through \textit{step-attributed rubric judging} and \textit{cross-rollout step normalization}, each reasoning step receives a step-wise rubric reward, while a \textit{decoupled advantage estimator} keeps RL training stable against rubric noise. On six mathematical benchmarks, \textit{SRaR} outperforms the state-of-the-art by \textbf{3.57} points on Qwen3-8B and \textbf{2.75} points on Qwen3-32B, raises the Faithful Reasoning Rate from 34.5\% to 46.7\%, and reduces self-correction looping from 48.1\% to 26.5\%. We hope \textit{SRaR} provides a practical paradigm for delivering fine-grained rubric supervision to RL-trained reasoners, showing that the granularity at which supervision is applied matters as much as its content. We discuss limitations and future directions in Appendix~\ref{app:limitations}.

%% file: appendix.tex
\beginappendix

\section{Notation}
\label{app:notation}

Table~\ref{tab:notation} summarizes the notation used throughout the paper.

\begin{table}[h]
\centering
\caption{Summary of notation.}
\label{tab:notation}
\small
\begin{tabular}{ll}
\toprule
Symbol & Description \\
\midrule
$x \in \mathcal{X}$ & Problem prompt; $\mathcal{X}$ denotes the space of problem prompts \\
$y = (y_1, \ldots, y_T) \in \mathcal{V}^T$ & Response token sequence of length $T$ over the vocabulary $\mathcal{V}$ \\
$\mathcal{D} = \{(x, a^*, \mathcal{R}_x)\}$ & Training set of problems with answers and rubrics \\
$a^*$ & Ground-truth answer of problem $x$ \\
$\mathcal{R}_x = \{(c_j, \tau_j)\}_{j=1}^{M}$ & Rubric of problem $x$ with $M$ items \\
$c_j$ & Criterion text of rubric item $j$ \\
$\tau_j$ & Type of rubric item $j$, one of \texttt{SUGGEST}, \texttt{PITFALL}, \texttt{BONUS} \\
$\pi_\theta$ & Policy model with parameters $\theta$ \\
$n$ & Number of rollouts sampled per prompt (group size) \\
$r_i$ & Scalar outcome reward of rollout $i$ \\
$A_i$ & Group-normalized (GRPO) advantage of rollout $i$ \\
$K$ & Number of reasoning steps in a response \\
$(a_k, b_k)$ & Token span of step $k$ \\
$\mathcal{J}$ & LLM judge \\
$s_j \in \{0, 1\}$ & Whether rubric item $j$ is satisfied \\
$k_j \in \{1, \ldots, K\}$ & Step to which rubric item $j$ is attributed \\
$\delta_j$ & Signed reward delta of rubric item $j$ \\
$R_{\tau} \in \mathbb{R}$ & Signed reward budget for type $\tau$; $R_{\texttt{SUG}}, R_{\texttt{BON}} > 0$, $R_{\texttt{PIT}} < 0$ \\
$N_{\tau}$ & Number of rubric items of type $\tau$ in the current sample \\
$d_{k,i}$ & Raw rubric reward of step $k$ in rollout $i$ \\
$\mathcal{G}_k$ & Rollouts with at least one rubric item attributed to step $k$ \\
$\mu_k, \sigma_k$ & Mean and standard deviation of $d_{k,i}$ over $\mathcal{G}_k$ \\
$\bar{d}_{k,i}$ & Step-wise cross-rollout normalized rubric reward \\
$\tilde{r}_i^{(t)}$ & Token-level rubric offset at position $t$ of rollout $i$ \\
$r_{\text{acc},i}, r_{\text{fmt},i} \in \{0,1\}$ & Accuracy and format rewards of rollout $i$ \\
$\lambda$ & Mixing weight between accuracy and format rewards \\
$r_{\text{base},i}$ & Base outcome reward $(1-\lambda)\,r_{\text{acc},i} + \lambda\,r_{\text{fmt},i}$ \\
$\mu_{\text{base}}, \sigma_{\text{base}}$ & Mean and standard deviation of $r_{\text{base},i}$ across rollouts of the same prompt \\
$\hat{A}_i^{(t)}$ & Final token-level advantage at position $t$ of rollout $i$ \\
$\epsilon$ & Small constant for numerical stability \\
\bottomrule
\end{tabular}
\end{table}

\section{Limitations}
\label{app:limitations}

We identify the following limitations of our work:

\begin{enumerate}[leftmargin=1.5em, itemsep=4pt]
    \item \textbf{Dependence on the LLM judge.} SRaR relies on an external LLM judge to evaluate rubric satisfaction and attribute each item to a specific reasoning step. Errors in step attribution or satisfaction judgments propagate into the reward signal. Although step-wise cross-rollout normalization mitigates the impact of systematic biases, judge inaccuracies on individual samples may still introduce noise into training.

    \item \textbf{Additional inference cost.} Each training iteration requires invoking the LLM judge on every rollout, increasing wall-clock time and computational cost compared to standard GRPO or trajectory-level rubric methods. This cost scales linearly with the number of rollouts and rubric items per prompt.

    \item \textbf{Domain scope.} Our experiments cover mathematical reasoning and code generation (Appendix~\ref{app:code_domain}). While the framework is domain-agnostic in principle, we have not verified its effectiveness on other multi-step reasoning tasks such as scientific reasoning, tool use, factual multi-hop QA, or open-ended generation, where answer verification and step boundaries are less clean. The rubric generation pipeline may require domain-specific adaptation.

    \item \textbf{Rubric generation relies on a strong model.} The three-stage data pipeline uses a frontier model (GPT-5) to generate rollouts and distill rubrics. The quality of the resulting rubrics is bounded by the capabilities of this model, and the pipeline's applicability to settings where only weaker models are available remains unexplored.

    \item \textbf{Single-model rubric bias.} Because all rubrics are distilled from one strong model, they may be biased toward that model's preferred reasoning patterns and inherit its systematic omissions. We mitigate single-solution bias by sampling multiple independent solutions without a step-format instruction and asking the model to compare and synthesize them, so a rubric is not tied to the proof structure of any single sampled trajectory. Model-level bias may nonetheless remain, particularly when alternative valid proof strategies are underrepresented among the sampled solutions. Multi-model rubric generation and human validation are natural remedies that we leave to future work.

    \item \textbf{Step boundary detection.} SRaR assumes that the policy model produces responses with explicit step headers (\texttt{\#\#\# Step N}). This requires a format prompt during training and may not generalize to free-form reasoning without structural delimiters, and SRaR remains untested on less structured reasoning where no comparable segmentation interface exists.

    \item \textbf{Step atomicity is not enforced.} SRaR treats each formatted step as the unit of credit assignment but does not verify that a step corresponds to a single atomic reasoning operation. A policy could in principle pack several atomic operations into one formatted step; this would make the assigned reward coarser, since all tokens in that span share one step-level signal, and would therefore erode the benefit of fine-grained credit assignment rather than inflate the reward. An atomicity-aware judge or an auxiliary penalty on overly composite steps is a promising extension.
\end{enumerate}

\section{Prompt Details}
\label{sec:prompts_details}

This section provides the full prompts used in our data pipeline and training procedure.

\subsection{Rollout Generation Prompt}
\label{app:prompt_rollout}

\begin{AIbox}{Prompt for Rollout Generation}
Solve the following math problem step by step. Please strictly follow these formatting instructions:

\begin{enumerate}[leftmargin=1.5em, itemsep=2pt, parsep=0pt, topsep=4pt]
\item \textbf{Step Format:} Organize your reasoning into clear steps. Start each step with the header ``\#\#\# Step N:'' (e.g., \#\#\# Step 1:, \#\#\# Step 2:).
\item \textbf{Final Answer:} The last line of your response must be enclosed in tags like this: ``$<$answer$>$Answer$<$/answer$>$'', where Answer is the calculated result.
\end{enumerate}

Problem: \texttt{\{Problem\}}
\end{AIbox}

\subsection{Rubric Distillation Prompt}
\label{app:prompt_rubric}

\begin{AIbox}{Rubric Distillation Prompt Template}
You are designing a grading rubric for math problem solutions.

\noindent The original problem statement is: \texttt{\{problem\_text\}}

\noindent You are given n model solutions with a correctness label. For each numbered solution: \texttt{\{solutions\_block\}}

\noindent The true correct answer for this problem is: $<$ANSWER$>$\texttt{\{gold\_answer\}}$<$/ANSWER$>$

\noindent Based on the samples and the true answer, create a concise, verifiable rubric. Output ONLY structured items using the tags below (one item per line):

\begin{itemize}[leftmargin=1.5em, itemsep=2pt, parsep=0pt, topsep=4pt]
    \item \textbf{$<$SUGGEST$>$}: Mandatory. One scoring checkpoint per line. Each checkpoint checks whether the candidate CoT covers a key reasoning step from the standard solution path. Make it specific and objectively checkable from the text.
    \item \textbf{$<$PITFALL$>$}: Optional. Only include if you are confident. Flags reasoning traps or common errors to avoid.
    \item \textbf{$<$BONUS$>$}: Optional. Only include if you are confident. Rewards especially clever, efficient, or elegant reasoning patterns.
    \item \textbf{$<$ANSWER$>$}: Mandatory. One line that checks whether the final answer matches the ground truth above (allow equivalent numeric forms).
\end{itemize}

\textbf{Guidelines:}
\begin{itemize}[leftmargin=1.5em, itemsep=2pt, parsep=0pt, topsep=2pt]
    \item Select only the most essential scoring points; total lines $\leq$ 8 (SUGGEST/ANSWER mandatory; PITFALL/BONUS only if sure). Do NOT pad to the limit; stop once the main reasoning chain is covered.
    \item Each $<$SUGGEST$>$ is exactly one scoring point and must be independently verifiable by an LLM judge from the candidate response.
    \item Do NOT include formatting or output-structure checks (e.g., headers, tags, markdown). Focus solely on mathematical / logical content steps and correctness of the final answer.
    \item Do not reference the given samples explicitly; generalize the rubric.
    \item Keep wording concise and actionable; no extra commentary before or after.
\end{itemize}
\end{AIbox}

\subsection{Training Format Prompt}
\label{app:prompt_format}

\begin{AIbox}{Format Prompt}
Solve the following math problem step by step. Follow these formatting rules:

\vspace{6pt}
\begin{enumerate}[leftmargin=1.5em, itemsep=2pt, parsep=0pt, topsep=2pt]
    \item \textbf{Steps:} Break your solution into multiple clear steps. Begin each step with ``\texttt{\#\#\# Step N:}'' (e.g., \texttt{\#\#\# Step 1:}, \texttt{\#\#\# Step 2:}, \texttt{\#\#\# Step 3:}).
    \item \textbf{Final Answer:} End your response with the answer inside \texttt{\textbackslash boxed\{\}}.
\end{enumerate}

\vspace{6pt}
Problem: \texttt{\{Problem\}}
\end{AIbox}

\subsection{LLM Judge Prompt}
\label{app:prompt_judge}

\begin{AIbox}{LLM Judge Prompt}
\textbf{Problem}\\
\texttt{\{problem\}}

\vspace{6pt}
\textbf{Rubric Items}\\
\texttt{\{rubric\_items\}}

\vspace{6pt}
\textbf{Student Solution}\\
\texttt{\{response\}}

\vspace{6pt}
\textbf{Student's Final Answer}\\
\texttt{\{extracted\_answer\}}

\vspace{6pt}
\textbf{Task}\\
For every rubric item listed above, evaluate the student's solution.

\textbf{Rules:}
\begin{itemize}[leftmargin=1.5em, itemsep=2pt, parsep=0pt, topsep=2pt]
    \item SUGGEST $\rightarrow$ \texttt{satisfied}: true if the student correctly performed that reasoning step.
    \item PITFALL $\rightarrow$ \texttt{satisfied}: true if the student made that mistake (fell into the pitfall).
    \item BONUS $\rightarrow$ \texttt{satisfied}: true if the student used that approach.
\end{itemize}

\noindent For \texttt{step}: the 1-indexed step number (\#\#\# Step N) most closely associated with this item. Use 0 if it spans the whole solution; use -1 if there is no relevant step.

\vspace{6pt}
\noindent Return ONLY a valid JSON array:\\
\texttt{[}\\
\hspace*{1.5em}\texttt{\{"id": $<$int$>$, "satisfied": $<$bool$>$, "step": $<$int$>$\},}\\
\hspace*{1.5em}\texttt{...}\\
\texttt{]}
\end{AIbox}

\subsection{Step-Level Judge Prompt}
\label{app:prompt_step_judge}

\begin{AIbox}{Step-Level Correctness Judge Prompt}
You are a math reasoning step verifier. Given a math problem and a sequence of reasoning steps, determine whether the CURRENT STEP is logically correct.

\vspace{6pt}
\textbf{Problem}\\
\texttt{\{question\}}

\vspace{6pt}
\textbf{Ground Truth Answer}\\
\texttt{\{answer\}}

\vspace{6pt}
\textbf{Previous Steps (assumed correct for this evaluation)}\\
\texttt{\{previous\_steps\}}

\vspace{6pt}
\textbf{Current Step to Evaluate}\\
\texttt{\{current\_step\_title\}}\\
\texttt{\{current\_step\_content\}}

\vspace{6pt}
\textbf{Instructions}\\
Analyze whether the current step contains correct reasoning, calculations, and logic. Consider:
\begin{enumerate}[leftmargin=1.5em, itemsep=2pt, parsep=0pt, topsep=2pt]
    \item Are the mathematical operations correct?
    \item Is the logic/deduction valid?
    \item Are there any errors in the step?
\end{enumerate}

\vspace{4pt}
\noindent Respond with EXACTLY one of:
\begin{itemize}[leftmargin=1.5em, itemsep=2pt, parsep=0pt, topsep=2pt]
    \item ``CORRECT'' if the step is logically and mathematically correct
    \item ``INCORRECT'' if the step contains any error
\end{itemize}
\end{AIbox}

\section{Training Hyperparameters}
\label{app:hparams}

\begin{table}[h]
\centering
\caption{Training Hyperparameters.}
\label{tab:hparams}
\begin{tabular}{lc}
\toprule
\textbf{Hyperparameter} & \textbf{Value} \\
\midrule
Max Prompt Length & 2048 \\
Max Response Length & 8192 \\
Train Batch Size & 128 \\
Learning Rate & $1 \times 10^{-6}$ \\
Rollout Group Size ($n$) & 8 \\
Training Steps & 200 \\
Temperature & 1.0 \\
\bottomrule
\end{tabular}
\end{table}

\section{Training Dynamics}
\label{app:training_dynamics}

\begin{figure}[h]
    \centering
    \includegraphics[width=\textwidth]{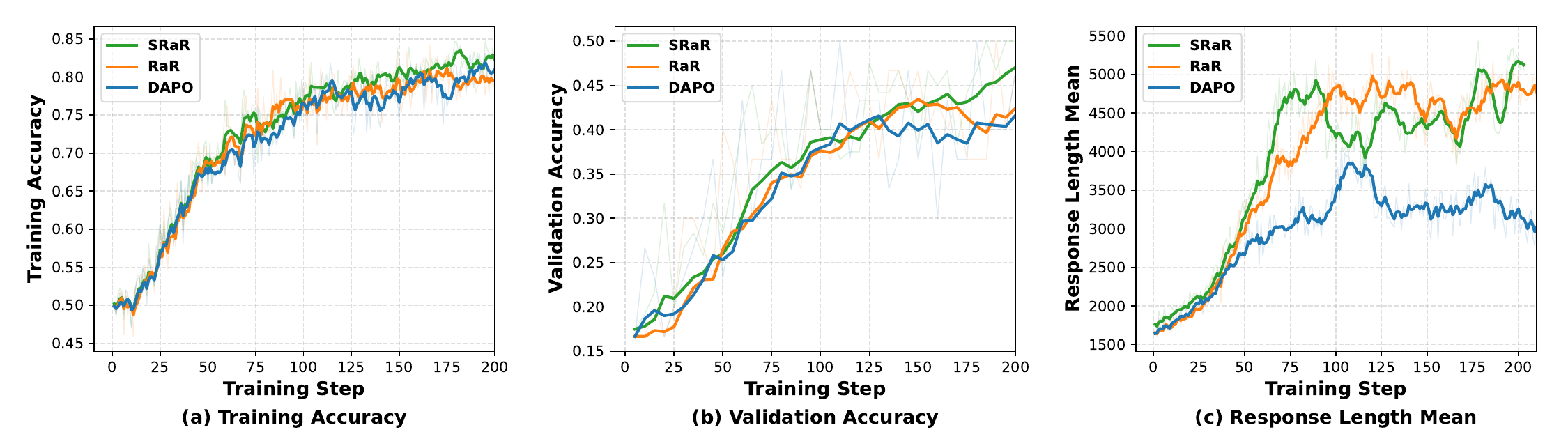}
    \caption{Training dynamics of DAPO, RaR, and SRaR on Qwen3-32B-Non-Thinking. \textbf{(a)} Training reward accuracy. \textbf{(b)} Validation accuracy. \textbf{(c)} Mean response length over training.}
    \label{fig:training_dynamics}
\end{figure}

Figure~\ref{fig:training_dynamics} illustrates the training dynamics of different methods on Qwen3-32B-Non-Thinking.
As shown in Figure~\ref{fig:training_dynamics}(a), all three methods converge to similar training reward accuracy (${\sim}0.80$--$0.83$), with SRaR reaching marginally higher values, indicating that step-level credit assignment provides effective gradient signals even at larger scale.
Figure~\ref{fig:training_dynamics}(b) reports validation accuracy. SRaR consistently achieves the highest validation accuracy throughout training, reaching ${\sim}0.47$ at convergence compared to RaR (${\sim}0.43$) and DAPO (${\sim}0.42$), confirming that the training gains translate to held-out performance.
Meanwhile, Figure~\ref{fig:training_dynamics}(c) shows that both RaR and SRaR produce longer responses than DAPO, with similar length trajectories. Despite comparable response lengths, SRaR achieves substantially higher accuracy (Figure~\ref{fig:training_dynamics}(b)), indicating that its gains stem from more effective reasoning per token rather than simply generating more content. The per-aspect rubric score dynamics (Suggest, Bonus, and Pitfall) are provided in Figure~\ref{fig:rubric_scores}.

\section{Format Compliance During Training}
\label{app:format_compliance}

\begin{figure}[H]
    \centering
    \includegraphics[width=\linewidth]{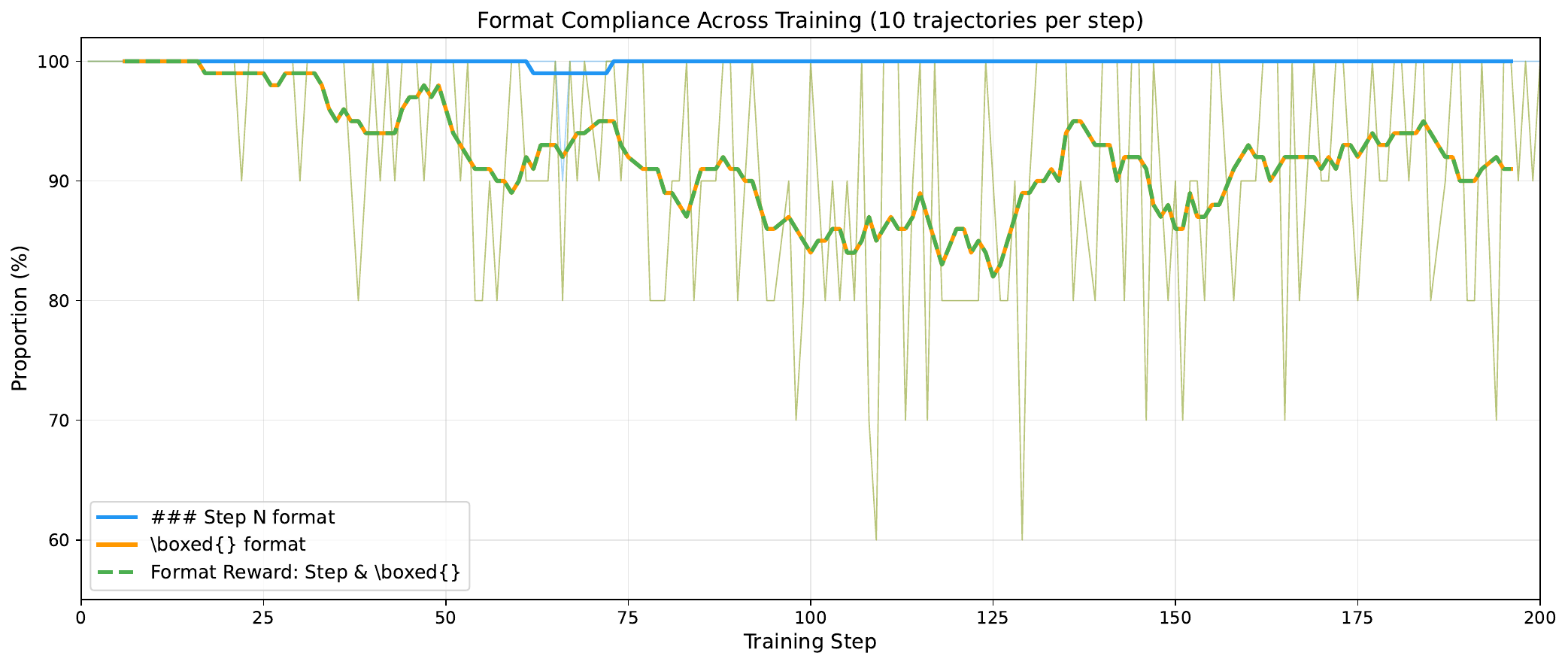}
    \caption{Format compliance across training steps for Qwen3-8B-Non-Thinking. At each training step, 10 rollouts are sampled and checked for (1)~\texttt{\#\#\# Step N} headers, (2)~\texttt{\textbackslash boxed\{\}} answer, and (3)~the conjunction of both (i.e., the format reward). The step-header format is maintained at near-100\% throughout training. The occasional drop in \texttt{\textbackslash boxed\{\}} compliance is attributable to response truncation at the maximum generation length, which prevents the model from emitting a final boxed answer.}
    \label{fig:format_compliance}
\end{figure}

\section{Rubric Quality Evaluation}
\label{app:rubric_quality}

To assess the quality of rubrics produced by our pipeline, we conducted an automated validity evaluation on 51,583 rubric items spanning $\sim$8,000 problems, using Gemini-3-Flash as the LLM judge. Each item was independently judged as valid or invalid with a brief justification.
\textbf{Overall validity.}
After this treatment, \textbf{98.0\%} of evaluated rubric items are judged valid, with only 1,027 items (2.0\%) flagged as invalid.
\textbf{Validity by rubric type.}
Table~\ref{tab:rubric_quality} shows validity rates broken down by rubric aspect. All three types achieve high and consistent validity rates in the range 96.3\%--98.5\%, confirming that the pipeline generates reliable rubrics across all aspects.

\begin{table}[h]
\centering
\caption{Rubric validity rates by type, evaluated over 51,583 items from $\sim$8,000 problems.}
\label{tab:rubric_quality}
\vspace{0.4em}
\small
\begin{tabular}{l c c c}
\toprule
\textbf{Rubric Type} & \textbf{Total} & \textbf{Valid} & \textbf{Validity Rate} \\
\midrule
\textsc{Suggest} & 36,532 & 35,994 & 98.5\% \\
\textsc{Pitfall} &  9,822 &  9,458 & 96.3\% \\
\textsc{Bonus}   &  5,213 &  5,089 & 97.6\% \\
\midrule
\textbf{Overall} & \textbf{51,583} & \textbf{50,556} & \textbf{98.0\%} \\
\bottomrule
\end{tabular}
\end{table}
\textbf{Human--AI agreement.}
Two of the authors independently annotated a 100-item cross-check sample using the same prompt below, achieving 92\% agreement with the LLM judge.

\begin{AIbox}{Rubric Validity Judge Prompt}
\small
You are a mathematics expert evaluating the quality of a rubric item for a math problem.

\noindent\textbf{Math Problem}

\noindent\texttt{\{problem\}}

\noindent\textbf{Ground Truth Answer}

\noindent\texttt{\{ground\_truth\}}

\noindent\textbf{Rubric Item}

\noindent Type: \texttt{\{rubric\_type\}}\\
Content: \texttt{\{rubric\_text\}}

\noindent\textbf{Rubric Type Definitions}
\begin{itemize}[leftmargin=1.5em, itemsep=2pt, parsep=0pt, topsep=4pt]
  \item \textbf{SUGGEST}: A recommended reasoning step or approach for solving the problem.
  \item \textbf{PITFALL}: A common mistake or error that students might make.
  \item \textbf{BONUS}: An alternative valid approach or insight.
  \item \textbf{ANSWER}: The expected final answer.
\end{itemize}

\noindent\textbf{Your Task}

\noindent Judge whether this rubric item is \textbf{valid} for the given math problem:
\begin{enumerate}[leftmargin=1.5em, itemsep=2pt, parsep=0pt, topsep=2pt]
  \item Is the content mathematically correct?
  \item Is it relevant to this specific problem?
  \item Is the type label (SUGGEST/PITFALL/BONUS/ANSWER) appropriate?
\end{enumerate}

\noindent Respond with ONLY a JSON object (no markdown, no extra text):\\
\texttt{\{"valid": true or false, "reason": "brief explanation"\}}
\end{AIbox}

\begin{figure}[h]
    \centering
    \includegraphics[width=\textwidth]{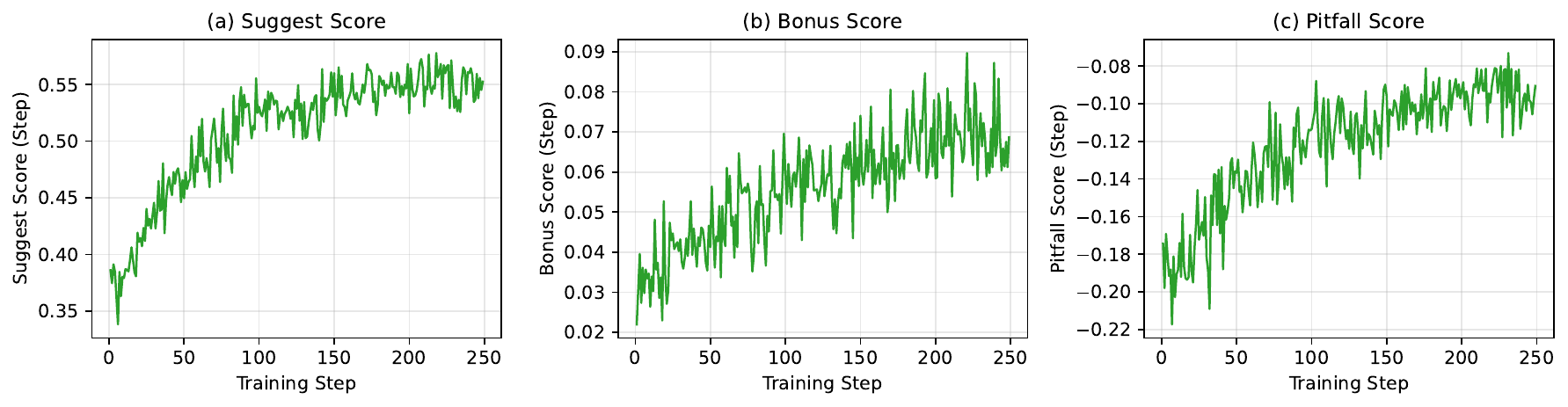}
    \caption{Step-level rubric score dynamics of SRaR during training on Qwen3-8B-Non-Thinking. \textbf{(a)} Suggest score steadily increases, indicating the model increasingly follows recommended reasoning steps. \textbf{(b)} Bonus score rises as the model discovers alternative solution strategies. \textbf{(c)} Pitfall score increases (toward zero), showing the model progressively avoids common errors.}
    \label{fig:rubric_scores}
\end{figure}

\section{LLM Judge Agreement Analysis}
\label{app:judge_agreement}

A core component of SRaR is the LLM judge that evaluates each rubric item against the model's response and attributes it to a specific reasoning step. To validate the reliability of this judge, we conduct an agreement study on 100 training trajectories sampled uniformly across training steps 20--200 (10 trajectories per checkpoint, totaling 664 rubric items).

\subsection*{Setup}
We use GPT-OSS-20B as the LLM judge (identical to the judge used during training) and compare its outputs against independent human annotations by two of the authors on the same 664 rubric items. For each item, both the LLM judge and the human annotators produce a binary satisfied judgment and a step attribution (the step number most closely associated with the rubric item).

\subsection*{Agreement on Rubric Satisfaction}

Table~\ref{tab:judge_agreement} reports the agreement between the two independent evaluations.

\begin{table}[h]
\centering
\caption{LLM judge agreement on 664 rubric items from 100 training trajectories.}
\label{tab:judge_agreement}
\vspace{0.4em}
\small
\begin{tabular}{l c}
\toprule
\textbf{Metric} & \textbf{Value} \\
\midrule
Satisfied agreement rate & 89.3\% \\
Satisfied Cohen's $\kappa$ & 0.76 \\
Step attribution agreement & 80.7\% \\
\bottomrule
\end{tabular}
\end{table}

The two evaluations agree on 89.3\% of rubric items with a Cohen's $\kappa$ of 0.76, corresponding to substantial agreement on the Landis--Koch scale. For step attribution, when both evaluations assign a specific step number, they agree 80.7\% of the time.

\subsection*{Judge Stochasticity}

The agreement above measures whether the judge matches human annotators. A separate question is whether the judge is self-consistent, since an unstable judge would inject noise into the reward regardless of its accuracy. We therefore re-evaluate each of 10 fixed trajectories 10 times with the same judge and prompt, and measure how often repeated evaluations of the same trajectory agree. Rubric-satisfaction judgments agree at a rate of 0.983 and step attributions at 0.997. Judge stochasticity is thus small enough that it cannot account for the differences between RaR and SRaR, which use identical rubrics and the same judge.

\subsection*{Rubric Scores and Answer Correctness}

We further analyze the relationship between rubric satisfaction rates and final answer correctness across the 100 trajectories (Table~\ref{tab:rubric_correctness}).

\begin{table}[h]
\centering
\caption{Rubric satisfaction rates for correct vs.\ incorrect trajectories (100 trajectories, 664 rubric items). Point-biserial correlations are between binary correctness and satisfaction rate.}
\label{tab:rubric_correctness}
\vspace{0.4em}
\small
\begin{tabular}{l c c c c c}
\toprule
\textbf{Rubric Type} & \textbf{Correct} & \textbf{Incorrect} & \textbf{Diff.} & \textbf{$r_{pb}$} & \textbf{$p$-value} \\
\midrule
\textsc{Suggest} satisfied & 74.2\% & 24.7\% & $+$49.5 pp & 0.62 & $<$0.001 \\
\textsc{Pitfall} triggered & 2.5\% & 24.2\% & $-$21.7 pp & $-$0.31 & 0.002 \\
\textsc{Bonus} satisfied & 12.5\% & 0.0\% & $+$12.5 pp & 0.28 & 0.005 \\
\bottomrule
\end{tabular}
\end{table}

All three rubric aspects are significantly correlated with answer correctness ($p < 0.01$). Correct trajectories satisfy 74.2\% of \textsc{Suggest} items (vs.\ 24.7\% for incorrect ones), trigger only 2.5\% of \textsc{Pitfall} items (vs.\ 24.2\%), and satisfy 12.5\% of \textsc{Bonus} items (vs.\ 0\%). This confirms that the rubric scores produced by the LLM judge are meaningful indicators of reasoning quality and are well-aligned with outcome correctness.

\section{Judge Model Robustness}
\label{app:judge_robustness}

To validate that SRaR is not dependent on a specific LLM judge, we replace the default GPT-OSS-20B judge with two alternative open-source models while keeping all other training hyperparameters identical: Qwen3-30B-A3B-Instruct, a judge of comparable capacity, and Qwen3-4B, a substantially weaker judge that probes how much judge capacity SRaR actually requires. Table~\ref{tab:judge_robustness} reports the results on Qwen3-8B-Non-Thinking.

\begin{table}[h]
\centering
\caption{SRaR performance on Qwen3-8B-Non-Thinking with different LLM judges. All three judges produce improvements over the RaR baseline (68.96 avg.), confirming that SRaR's gains are robust to the choice of judge model.}
\label{tab:judge_robustness}
\vspace{0.4em}
\resizebox{\textwidth}{!}{%
\begin{tabular}{l ccc cccc}
\toprule
\textbf{Judge Model} & \textbf{AIME24} & \textbf{AIME25} & \textbf{AMC23} & \textbf{MATH500} & \textbf{Minerva} & \textbf{Olympiad} & \textbf{Avg.} \\
\midrule
RaR (no step-level) & 57.50 & 43.23 & 88.20 & 92.6 & 64.71 & 67.51 & 68.96 \\
\midrule
SRaR w/ GPT-OSS-20B & \textbf{61.15} & \textbf{49.06} & \textbf{92.73} & \textbf{94.4} & 64.71 & \textbf{73.15} & \textbf{72.53} \\
SRaR w/ Qwen3-30B-A3B & 57.81 & 46.04 & 91.25 & 93.20 & \textbf{66.54} & 70.33 & 70.86 \\
SRaR w/ Qwen3-4B & \textbf{61.15} & 47.60 & 92.11 & \textbf{95.00} & 58.46 & 71.67 & 71.00 \\
\bottomrule
\end{tabular}%
}
\end{table}

With Qwen3-30B-A3B-Instruct as the judge, SRaR still achieves a $+1.90$ average improvement over RaR (70.86 vs.\ 68.96), demonstrating that the step-wise reward framework is effective with a different judge model. The gap relative to the GPT-OSS-20B judge ($-1.67$ avg.) is modest, suggesting that judge choice primarily affects the magnitude of improvement rather than the viability of the approach. This is consistent with the expectation that different judges produce varying step attributions, which are partially mitigated by cross-rollout normalization (noise that is uncorrelated across rollouts averages out during normalization).

Dropping to Qwen3-4B, a judge roughly an order of magnitude smaller than the default, costs only $1.53$ points on average (71.00 vs.\ 72.53) and still improves over RaR by $+2.04$. The degradation is not uniform across benchmarks: the weaker judge matches or exceeds the default judge on AIME24 (61.15) and MATH500 (95.00) while losing ground mainly on MinervaMath ($58.46$ vs.\ $64.71$). SRaR therefore does not require a particularly strong judge to be effective, though judge capacity does set the ceiling on how much of the available gain is realized.

\section{Rubric Examples}
\label{app:rubric_examples}

We present two representative examples from our curated training set to illustrate the structure and quality of the generated rubrics. Each problem is accompanied by a rubric list consisting of \textsc{Suggest}, \textsc{Pitfall}, \textsc{Bonus}, and \textsc{Answer} items produced by our three-stage pipeline.

\subsection*{Example 1}

\textbf{Problem.} Given that nonzero real numbers $x$ and $y$ satisfy $x + \frac{1}{y} = 3$ and $y + \frac{1}{x} = 4$, find the value of $xy + \frac{1}{xy}$.

\textbf{Ground Truth.} 10

\begin{RubricCaseBox}{Generated Rubric}
\small
\texttt{<SUGGEST>} Multiplies the two equations to obtain $\left(x+\tfrac{1}{y}\right)\!\left(y+\tfrac{1}{x}\right)=12$.

\texttt{<SUGGEST>} Expands the product correctly to $xy + 2 + \tfrac{1}{xy}$.

\texttt{<SUGGEST>} From $xy + 2 + \tfrac{1}{xy} = 12$, isolates and states $xy + \tfrac{1}{xy} = 10$.

\texttt{<PITFALL>} Misexpanding $\left(x+\tfrac{1}{y}\right)\!\left(y+\tfrac{1}{x}\right)$, especially omitting the $+2$ term from $x\cdot\tfrac{1}{x}$ and $\tfrac{1}{y}\cdot y$.

\texttt{<BONUS>} Computes $xy + \tfrac{1}{xy}$ directly via the product expansion without solving for $x$ or $y$ individually.

\texttt{<ANSWER>} Final value reported for $xy + \tfrac{1}{xy}$ equals $10$.
\end{RubricCaseBox}

\subsection*{Example 2}

\textbf{Problem.} Let $S$ be the set of triples $(a,b,c)$ of non-negative integers such that $a+b+c$ is even. The value of $\sum_{(a,b,c)\in S} \frac{1}{2^a 3^b 5^c}$ can be expressed as $\frac{m}{n}$ where $\gcd(m,n)=1$. Compute $m+n$.

\textbf{Ground Truth.} 37

\begin{RubricCaseBox}{Generated Rubric}
\small
\texttt{<SUGGEST>} Applies the parity indicator $\frac{1+(-1)^{a+b+c}}{2}$ to rewrite the restricted sum as $\frac{S_0+S_1}{2}$, where $S_0$ is the unrestricted sum and $S_1$ includes the factor $(-1)^{a+b+c}$.

\texttt{<SUGGEST>} Factors $S_0$ into $\sum \tfrac{1}{2^a}\cdot\sum\tfrac{1}{3^b}\cdot\sum\tfrac{1}{5^c}$ and evaluates to $2\cdot\tfrac{3}{2}\cdot\tfrac{5}{4}=\tfrac{15}{4}$.

\texttt{<SUGGEST>} Factors $S_1$ into $\sum\tfrac{(-1)^a}{2^a}\cdot\sum\tfrac{(-1)^b}{3^b}\cdot\sum\tfrac{(-1)^c}{5^c}$ and evaluates to $\tfrac{2}{3}\cdot\tfrac{3}{4}\cdot\tfrac{5}{6}=\tfrac{5}{12}$.

\texttt{<SUGGEST>} Combines via $\frac{S_0+S_1}{2}=\frac{15/4+5/12}{2}=\frac{25}{12}$, giving $m+n=37$.

\texttt{<PITFALL>} Omits the $\tfrac{1}{2}$ factor in the parity projection, or uses $\frac{S_0-S_1}{2}$ (which sums over odd $a+b+c$) instead of $\frac{S_0+S_1}{2}$.

\texttt{<BONUS>} Uses the compact form $\frac{1}{2}\!\left(\frac{1}{(1-x)(1-y)(1-z)}+\frac{1}{(1+x)(1+y)(1+z)}\right)$ with $x=\tfrac{1}{2},y=\tfrac{1}{3},z=\tfrac{1}{5}$.

\texttt{<ANSWER>} Final reported $m+n$ equals $37$.
\end{RubricCaseBox}

\section{Case Studies}
\label{app:case_studies}

We provide representative case studies from AIME~2025 (repeat\_id=0) to illustrate the qualitative differences between RaR and SRaR identified in our analysis.

\subsection{Process--Outcome Alignment Case Study}
\label{app:case_process_outcome}

We present two problems where RaR produces incorrect answers with step-level errors, while SRaR produces correct answers with all steps judged correct.
\textbf{Case 1: AIME~2025 Problem 20.}
Circle $\omega_1$ with radius 6 centered at point $A$ is internally tangent at point $B$ to circle $\omega_2$ with radius 15\ldots The area of rectangle $EFGH$ is $\frac{m}{n}$\ldots Find $m+n$. (Ground truth: 293.)

\begin{CaseStudyBox}{{RaR Response (Problem 20) --- Incorrect, Looping}}
\small
\texttt{\#\#\# Step 1: Set up coordinates}

We can place the center of $\omega_2$ at the origin\ldots Since the circles are internally tangent at point $B$\ldots the center $A$ of $\omega_1$ lies along the line connecting the centers\ldots

\vspace{4pt}
[The model enters an extended self-revision loop with 34 revision phrases across 30{,}922 characters. It repeatedly reconsiders the geometric constraints without converging. The response is cut off mid-computation without producing a \texttt{\textbackslash boxed\{\}} answer.]

\vspace{4pt}
\texttt{...So the inequality becomes: 33 - p < p + 15\ldots 18}

[Response truncated --- no final answer produced.]
\end{CaseStudyBox}

\begin{CaseStudyBox}{{SRaR Response (Problem 20) --- Correct, All Steps Correct}}
\small
\texttt{\#\#\# Step 1: Understand the configuration and setup coordinates}

We are given two circles, $\omega_1$ and $\omega_2$\ldots Let's place the center of $\omega_2$ at the origin. Then $A = (9, 0)$, $O = (0, 0)$, $B = (15, 0)$.

\vspace{4pt}
\texttt{\#\#\# Step 2: Find key points}

Since $BC$ is a diameter of $\omega_2$, $C = (-15, 0)$. Since $BC \perp AD$, $D$ lies directly above $A$\ldots $D = (9, 12)$.

\vspace{4pt}
\texttt{\#\#\# Step 3: Compute the rectangle and verify equal areas}

\ldots area of the rectangle is $288/5$, so $m = 288$, $n = 5$, $m + n = 293$.

\vspace{4pt}
\texttt{\#\#\# Step 4: Final Answer}

\texttt{\textbackslash boxed\{293\}}
\end{CaseStudyBox}

\noindent\textbf{Analysis.} RaR's trajectory-level reward fails to prevent the model from entering an extended self-revision loop: it repeatedly reconsiders the geometric constraints without converging, ultimately exhausting the generation budget across 30{,}922 characters. In contrast, SRaR's step-wise reward allows the model to commit to each step and progress efficiently through the solution in four clean steps with only 6{,}640 characters.
\textbf{Case 2: AIME~2025 Problem 24.}
Sixteen chairs are arranged in a row. Eight people each select a chair in which to sit so that no person sits next to two other people. Let $N$ be the number of subsets\ldots Find the remainder when $N$ is divided by 1000. (Ground truth: 907.)

\begin{CaseStudyBox}{{RaR Response (Problem 24) --- Incorrect, Unstructured Looping}}
\small
We are given 16 chairs arranged in a row\ldots Wait, actually the problem says ``no person sits next to two other people.'' Hmm, let me parse that carefully.

\vspace{4pt}
[The model never produces any \texttt{\#\#\# Step} headers. Instead, it generates an unstructured stream of reasoning with 63 self-correction phrases across 33{,}378 characters. It repeatedly restarts the dynamic programming computation without converging. The response ends mid-computation:]

\vspace{4pt}
\texttt{...k=6: dp[7][6] + dp[6][5] + dp[5][4] = 0 + 0 + 1 = 1}\\
\texttt{k=7: 0}\\
\texttt{k=8: 0}\\
\texttt{n=8:}\\
\texttt{k=0: 1}\\
\texttt{k=1: 8}

[Response truncated --- no final answer produced.]
\end{CaseStudyBox}

\begin{CaseStudyBox}{{SRaR Response (Problem 24) --- Correct, All Steps Correct}}
\small
\texttt{\#\#\# Step 1: Understand the problem}

We have 16 chairs arranged in a row. Eight people each select a chair\ldots The condition means no person is adjacent to two other people, i.e., no three consecutive chairs are occupied.

\vspace{4pt}
\texttt{\#\#\# Step 2: Count the valid subsets}

\ldots our calculation of 2907 should be correct. Thus the answer is 907.

\vspace{4pt}
\texttt{\#\#\# Step 3: Compute the final answer}

The remainder when $N$ is divided by 1000 is \texttt{\textbackslash boxed\{907\}}.
\end{CaseStudyBox}

\noindent\textbf{Analysis.} Both models initially grapple with interpreting the constraint, but RaR degenerates into an unstructured stream of 63 self-corrections without ever organizing its reasoning into steps, ultimately truncating without an answer. SRaR organizes the solution into three clear steps, correctly interprets the constraint, and arrives at the right answer.

\subsection{Looping Case Study}
\label{app:case_looping}

We present AIME~2025 Problem~13 as an illustrative example of the looping behavior analyzed in Section~\ref{sec:looping_analysis}.

Let $ABCDE$ be a convex pentagon with $AB=14$, $BC=7$, $CD=24$, $DE=13$, $EA=26$, and $\angle B=\angle E=60^{\circ}$. For each point $X$ in the plane, define $f(X)=AX+BX+CX+DX+EX$. The least possible value of $f(X)$ can be expressed as $m+n\sqrt{p}$\ldots Find $m+n+p$. (Ground truth: 60.)

\begin{CaseStudyBox}{{RaR Response (Problem 13) --- Severe Looping, No Answer}}
\small
\texttt{\#\#\# Step 1: Understanding the Problem}

The function $f(X)$ represents the sum of distances from a variable point $X$ to each of the five vertices\ldots This is the Fermat-Torricelli problem generalized to five points\ldots

\vspace{4pt}
[The model enters a revision loop with 31 self-correction phrases across 23{,}874 characters. It repeatedly restarts the coordinate geometry computation without converging. Representative excerpts:]

\vspace{4pt}
\texttt{...Wait, if DA is (a, b), then perpendicular vectors are (-b, a) and (b, -a)\ldots}\\
\texttt{...Wait, no: Wait, if DA is (a, b)\ldots}

\vspace{4pt}
[The response ends mid-sentence without producing a \texttt{\textbackslash boxed\{\}} answer.]
\end{CaseStudyBox}

\begin{CaseStudyBox}{{SRaR Response (Problem 13) --- Structured, Produces Answer}}
\small
\texttt{\#\#\# Step 1: Understand the problem and identify key elements}

We are given a convex pentagon $ABCDE$ with specific side lengths and angles at $B$ and $E$ each equal to $60^{\circ}$. We need to find the minimum of $f(X) = AX + BX + CX + DX + EX$\ldots

\vspace{4pt}
\texttt{\#\#\# Step 2: Apply geometric reasoning}

\ldots the total minimal value might be $7\sqrt{3} + 13\sqrt{3} = 20\sqrt{3}$ plus other terms\ldots

[The model works through the problem with self-correction (11 phrases), but stays structured and converges to an answer.]

\vspace{4pt}
\texttt{\textbackslash boxed\{58\}} (incorrect, but a complete structured response)
\end{CaseStudyBox}

\noindent\textbf{Analysis.} This problem is difficult for both models (neither answers correctly), but the behavioral contrast is stark. RaR produces a response with 31 self-correction phrases, repeatedly restarting its coordinate computation and ultimately failing to produce any answer across 23{,}874 characters. SRaR, guided by step-wise rewards, maintains a clear two-step structure and converges to a definite (though incorrect) answer in 15{,}483 characters. This illustrates how step-wise rewards act as a structural prior that prevents degeneration into aimless revision, even on problems the model cannot ultimately solve.

\section{Data-Size Sensitivity}
\label{app:data_size_sensitivity}

We analyze how model performance changes with the amount of synthetic data. Instead of varying task
types, we vary only the number of synthetic rubrics training samples. Concretely, we train for exactly one epoch
on subsets of size 1k, 2k, 4k, 8k, and 16k. All subsets are obtained by random
sampling with a fixed seed, so improvements mainly reflect
data amount rather than selection artifacts.

\begin{figure}[h]
    \centering
    \includegraphics[width=0.75\linewidth]{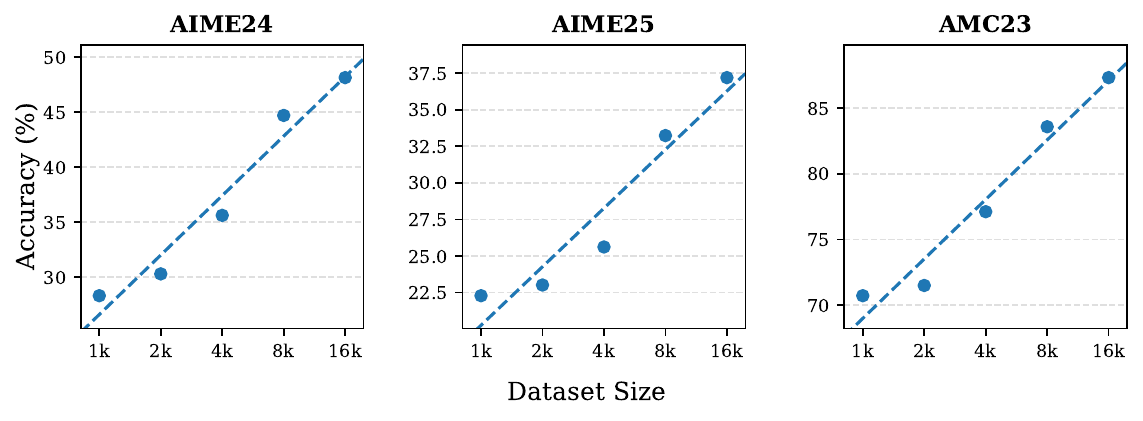}
    \caption{\textbf{Data-Size Effects of Synthetic Rubrics.} We report performance trends on AIME~2024, AIME~2025, and AMC~2023 as the amount of synthetic rubrics training data increases. ``1k/2k/4k/8k/16k'' are random subsets sampled with a fixed seed. The training recipe matches the main experimental setting. Results show a consistent upward trend across all benchmarks as synthetic data increases.}
    \label{fig:data_size_sensitivity}
\end{figure}

Figure~\ref{fig:data_size_sensitivity} reports accuracy across three benchmarks as data size scales from 1k to 16k.
Several observations stand out.
First, performance improves monotonically on every benchmark: even the smallest 1k subset already yields meaningful gains, while each doubling of data continues to push accuracy higher, confirming that our synthetic rubrics pipeline produces consistently useful training signal.
Second, the steepest improvement occurs between 4k and 8k samples. For instance, AIME25 jumps from 25.62\% to 33.23\%, and AMC23 rises from 77.11\% to 83.59\%, suggesting a critical mass of rubric diversity is reached in this regime.
Third, at the full 16k scale, the curves have not yet plateaued. For example, AIME24 reaches 48.12\%, AIME25 reaches 37.19\%, and AMC23 reaches 87.34\%, indicating that further scaling the synthetic data is likely to yield additional gains.
Finally, the trend is consistent across benchmarks of different difficulty levels, from the easier AMC23 (70.70\%$\to$87.34\%) to the harder AIME25 (22.29\%$\to$37.19\%), demonstrating that the benefit of more synthetic rubrics is not confined to a narrow difficulty band.

\section{Rubric Aspect Ablation}
\label{app:rubric_ablation}

We ablate the rubric aspects to disentangle the contribution of step-wise credit assignment from the rubric design. Table~\ref{tab:rubric_ablation} compares RaR and SRaR under four rubric configurations on Qwen3-8B-Non-Thinking: three single-aspect settings that retain only SUGGEST, only PITFALL, or only BONUS items, and the full multi-aspect set. Within each configuration, RaR and SRaR use exactly the same rubric items, judge, training data, and training recipe.

\begin{table}[h]
\centering
\caption{Rubric aspect ablation on Qwen3-8B-Non-Thinking. $\Delta$ denotes SRaR improvement over RaR within the same rubric configuration. Best results per configuration are \textbf{bolded}.}
\label{tab:rubric_ablation}
\vspace{0.5em}
\resizebox{\textwidth}{!}{%
\begin{tabular}{l ccc cccc}
\toprule
\textbf{Method} & \textbf{AIME24} & \textbf{AIME25} & \textbf{AMC23} & \textbf{MATH500} & \textbf{Minerva} & \textbf{Olympiad} & \textbf{Avg.} \\
\midrule
\rowcolor{gray!15}
\multicolumn{8}{l}{\textit{SUGGEST-only rubrics}} \\
RaR & 53.23 & 41.98 & \textbf{91.64} & 93.2 & 61.40 & 69.73 & 68.53 \\
\textbf{SRaR} & \textbf{58.65} & \textbf{46.98} & 91.56 & \textbf{93.6} & \textbf{66.18} & \textbf{71.36} & \textbf{71.39} \\
$\Delta$ & \textcolor{green!50!black}{+5.42} & \textcolor{green!50!black}{+5.00} & \textcolor{red}{--0.08} & \textcolor{green!50!black}{+0.40} & \textcolor{green!50!black}{+4.78} & \textcolor{green!50!black}{+1.63} & \textcolor{green!50!black}{+2.86} \\
\midrule
\rowcolor{gray!15}
\multicolumn{8}{l}{\textit{PITFALL-only rubrics}} \\
RaR & 52.08 & 39.79 & 87.50 & 93.8 & \textbf{66.18} & 66.77 & 67.69 \\
\textbf{SRaR} & \textbf{60.63} & \textbf{49.06} & \textbf{92.34} & \textbf{94.2} & 62.14 & \textbf{73.45} & \textbf{71.97} \\
$\Delta$ & \textcolor{green!50!black}{+8.55} & \textcolor{green!50!black}{+9.27} & \textcolor{green!50!black}{+4.84} & \textcolor{green!50!black}{+0.40} & \textcolor{red}{--4.04} & \textcolor{green!50!black}{+6.68} & \textcolor{green!50!black}{+4.28} \\
\midrule
\rowcolor{gray!15}
\multicolumn{8}{l}{\textit{BONUS-only rubrics}} \\
RaR & 55.21 & 40.73 & 88.05 & 93.0 & \textbf{63.97} & 67.06 & 68.00 \\
\textbf{SRaR} & \textbf{59.38} & \textbf{47.08} & \textbf{92.89} & \textbf{95.0} & 60.67 & \textbf{70.33} & \textbf{70.89} \\
$\Delta$ & \textcolor{green!50!black}{+4.17} & \textcolor{green!50!black}{+6.35} & \textcolor{green!50!black}{+4.84} & \textcolor{green!50!black}{+2.00} & \textcolor{red}{--3.30} & \textcolor{green!50!black}{+3.27} & \textcolor{green!50!black}{+2.89} \\
\midrule
\rowcolor{gray!15}
\multicolumn{8}{l}{\textit{Full rubrics (SUGGEST + PITFALL + BONUS)}} \\
RaR & 57.50 & 43.23 & 88.20 & 92.6 & \textbf{64.71} & 67.51 & 68.96 \\
\textbf{SRaR} & \textbf{61.15} & \textbf{49.06} & \textbf{92.73} & \textbf{94.4} & \textbf{64.71} & \textbf{73.15} & \textbf{72.53} \\
$\Delta$ & \textcolor{green!50!black}{+3.65} & \textcolor{green!50!black}{+5.83} & \textcolor{green!50!black}{+4.53} & \textcolor{green!50!black}{+1.80} & +0.00 & \textcolor{green!50!black}{+5.64} & \textcolor{green!50!black}{+3.57} \\
\bottomrule
\end{tabular}%
}
\end{table}

Two conclusions emerge. First, SRaR outperforms RaR under every rubric configuration (+2.86 with SUGGEST-only, +4.28 with PITFALL-only, +2.89 with BONUS-only, +3.57 with full rubrics), confirming that the primary source of improvement is step-wise credit assignment rather than the specific rubric aspects. Second, each aspect provides meaningful supervision on its own (SRaR reaches 71.39, 71.97, and 70.89 with SUGGEST, PITFALL, and BONUS items alone), and the full combination is best at 72.53, indicating that the three aspects are complementary rather than redundant. The single-aspect variants trade benchmarks against one another: PITFALL-only and BONUS-only both give up ground on MinervaMath relative to their RaR counterparts while gaining substantially on AIME24/AIME25, which is consistent with error avoidance and exceptional-insight credit being sharper but narrower signals than the full multi-aspect schema.

\paragraph{Role of BONUS relative to SUGGEST.} SUGGEST and BONUS carry the same positive sign in Eq.~\eqref{eq:delta} because both mark desirable reasoning behaviour, but they cover different content. SUGGEST items are the core reasoning checkpoints expected of any standard valid solution, such as deriving a necessary intermediate result or carrying out an essential calculation. BONUS items are optional and reward reasoning that goes beyond the standard solution path: applying an advanced result, recognizing a compact identity, constructing a substantially shorter derivation, or supplying an alternative solution. BONUS credit is evaluated independently of final-answer correctness by design, since a response may use an elegant intermediate argument correctly and then make an unrelated error later; SRaR rewards the locally correct BONUS step in that case, which is precisely the behaviour trajectory-level aggregation cannot express.

\section{Method Design Ablation}
\label{app:component_ablation}

We ablate three key design choices of SRaR on Qwen3-8B-Non-Thinking to isolate their individual contributions. Table~\ref{tab:component_ablation} reports results across all six benchmarks.

\begin{table}[h]
\centering
\caption{Component ablation on Qwen3-8B-Non-Thinking. Each variant removes one component from the full SRaR system. Best results are \textbf{bolded}.}
\label{tab:component_ablation}
\vspace{0.5em}
\resizebox{\textwidth}{!}{%
\begin{tabular}{l ccc cccc}
\toprule
\textbf{Method} & \textbf{AIME24} & \textbf{AIME25} & \textbf{AMC23} & \textbf{MATH500} & \textbf{Minerva} & \textbf{Olympiad} & \textbf{Avg.} \\
\midrule
\rowcolor{gray!15}
Base (Qwen3-8B, w/ step format) & 26.77 & 22.40 & 71.88 & 85.0 & 55.51 & 56.53 & 53.02 \\
DAPO (w/o step format) & 55.73 & 41.25 & 87.66 & 92.0 & 61.76 & 68.69 & 67.85 \\
SRaR (w/o step norm.) & 45.10 & 28.96 & 73.52 & 87.2 & 49.26 & 56.53 & 56.76 \\
SRaR (w/o decoupled adv.) & \textbf{62.60} & 46.67 & \textbf{93.28} & \textbf{94.6} & \textbf{64.71} & 70.47 & 72.06 \\
\rowcolor{lightred}
\textbf{SRaR (full)} & 61.15 & \textbf{49.06} & 92.73 & 94.4 & \textbf{64.71} & \textbf{73.15} & \textbf{72.53} \\
\bottomrule
\end{tabular}%
}
\end{table}
\textbf{Step-structured format prompt.}
Comparing DAPO trained without step format constraints (67.85 avg) against the full DAPO baseline in Table~\ref{tab:main_results_8b} (67.26 avg), the two are virtually identical, indicating that requiring the model to organize its chain-of-thought into explicit \texttt{\#\#\# Step} blocks is a benign structural constraint that introduces no measurable performance penalty. This makes it a practical prerequisite for SRaR, which relies on step boundaries to attribute rubric items to individual reasoning steps.
\textbf{Cross-rollout step normalization.}
Removing step normalization causes severe training collapse (56.76 avg), performing even worse than the base model on several benchmarks. Without normalization, rubric items that are uniformly satisfied across all rollouts still produce large positive advantages. The model exploits this by collapsing all reasoning into a single \texttt{\#\#\# Step 1}, maximizing rubric reward without distributing reasoning across multiple steps. We provide a representative example in Appendix~\ref{app:case_collapse}. Cross-rollout normalization prevents this by zeroing out the advantage for rubric items whose scores do not vary across rollouts, ensuring that only genuinely discriminative rubric items drive learning.
\textbf{Decoupled advantage estimation.}
Instead of maintaining separate outcome and rubric advantage streams, this variant computes a single RaR-style scalar reward (combining outcome correctness and rubric satisfaction) and then allocates it to the tokens of steps that received rubric scores. This approach yields competitive performance (72.06 avg) but falls short of the full SRaR (72.53 avg). The gap is most pronounced on harder benchmarks (AIME~2025: 46.67 vs.\ 49.06), suggesting that merging outcome and rubric signals into one scalar before allocation loses the benefit of keeping the outcome-driven baseline stable.

\section{Case Study: Training Collapse without Step Normalization}
\label{app:case_collapse}

Below is a representative response from the SRaR (w/o step normalization) model after training collapse. The model places its entire chain-of-thought (spanning over 4{,}000 tokens of deliberation) inside a single \texttt{\#\#\# Step~1}, never producing any subsequent step. This behavior emerges because, without cross-rollout normalization, rubric items that are consistently satisfied across all rollouts (e.g., ``identify the modular invariant'') still carry large positive advantages. By concentrating all reasoning in one step, the model ensures that this single step collects all rubric credit, removing any incentive to distribute reasoning across multiple steps.

\begin{FailCaseBox}{SRaR (w/o step norm.) -- Single-Step Collapse}
\small
\texttt{\#\#\# Step 1: Understand the problem and analyze the operations}

We are given a set of numbers from 1 to 2004 written on a desk. A move consists of selecting some numbers, calculating the remainder of their sum modulo 11, writing this remainder on the desk, and deleting the selected numbers. After several such moves, only two numbers remain on the desk, one of which is 1000. We are to find the other number.

First, let's understand the invariant\ldots\ Each move replaces some numbers with their sum modulo 11. So, if we consider the numbers on the desk as elements of the set, the key is that the sum of all numbers on the desk modulo 11 is preserved\ldots

\textit{[\ldots\ 4{,}000+ tokens of reasoning, all within Step 1, including repeated self-verification, alternative approaches, and redundant restatements\ldots]}

Therefore, the answer is 4.

\vspace{4pt}
\texttt{\#\#\# Final Answer}

\texttt{\textbackslash boxed\{4\}}
\end{FailCaseBox}

\noindent\textbf{Analysis.} The response contains a single reasoning step that spans the entire solution. While the final answer happens to be correct for this particular problem, the degenerate structure, including no step decomposition and excessive self-verification loops within one monolithic block, is a direct consequence of the un-normalized rubric reward. In contrast, the full SRaR model with cross-rollout normalization produces well-structured multi-step solutions, because normalization zeros out the advantage for rubric items that do not vary across rollouts, leaving only genuinely discriminative items to shape the policy gradient.

\section{Generalization to the Code Domain}
\label{app:code_domain}

The main experiments are confined to mathematical reasoning. To test whether step-wise rubric rewards transfer to a domain with different step semantics and a different verification mechanism, we repeat the full pipeline on code generation. We take 12K problems from Code-R1, generate problem-specific rubrics with exactly the same three-stage contrastive distillation pipeline used for math (Section~\ref{sec:data_pipeline}), and train DAPO, RaR, and SRaR under the same experimental setup, changing only the reward verifier to unit-test execution. Table~\ref{tab:code_domain} reports pass rates on four code benchmarks.

\begin{table}[h]
\centering
\caption{Code-domain results on Qwen3-8B-Non-Thinking, trained on 12K Code-R1 problems with rubrics generated by the same pipeline as the math corpus. Best results are \textbf{bolded}.}
\label{tab:code_domain}
\vspace{0.4em}
\resizebox{\textwidth}{!}{%
\begin{tabular}{l cccc}
\toprule
\textbf{Method} & \textbf{LiveCodeBench v6} & \textbf{HumanEval+} & \textbf{MBPP+} & \textbf{TACO} \\
\midrule
Qwen3-8B & 22.93 & 80.49 & 68.78 & 12.59 \\
+ DAPO & 40.07 & 87.80 & \textbf{77.25} & 30.69 \\
+ RaR & 40.93 & 87.80 & 75.13 & 29.52 \\
\rowcolor{lightred}
+ \textbf{SRaR (Ours)} & \textbf{41.93} & \textbf{87.98} & 76.60 & \textbf{30.72} \\
\bottomrule
\end{tabular}%
}
\end{table}

SRaR is best on LiveCodeBench~v6, HumanEval+, and TACO, and remains within $0.65$ points of the best method on MBPP+. The margin over trajectory-level RaR is consistent across all four benchmarks ($+1.00$, $+0.18$, $+1.47$, $+1.20$), which is the comparison that isolates step-wise credit assignment since both methods use identical rubrics, judge, data, and recipe. Notably, RaR fails to beat DAPO on MBPP+ and TACO whereas SRaR does on TACO, suggesting that trajectory-level rubric aggregation is more fragile than step-wise attribution when the reward signal is already dense from unit-test execution. Absolute gains are largest on LiveCodeBench~v6 and TACO, the two benchmarks with the longest multi-step solutions, consistent with the length-dependent trend observed on math (Section~\ref{sec:think_long_better}).

\section{Theoretical Properties of SRaR}
\label{app:theory}

We provide formal analyses of the key properties of SRaR. Throughout this section, we consider a single prompt $x$ with $n$ rollouts $\{y_1, \ldots, y_n\}$ sampled from the policy $\pi_\theta(\cdot | x)$. A small constant $\epsilon > 0$ is added to every normalization denominator for numerical stability; since each numerator is a mean-centred deviation, this does not affect the properties established below.

\subsection{Policy Gradient Decomposition}
\label{app:pg_decomp}

\begin{proposition}[Decomposition of the SRaR Policy Gradient]
\label{prop:pg_decomp}
Let $K_i$ denote the number of steps in rollout $i$, and let $\mathcal{T}_{k,i}$ denote the set of token indices belonging to step $k$ in rollout $i$. Under the SRaR advantage, the policy gradient estimate decomposes as:
\begin{equation}
    \hat{g}_{\textit{SRaR}} = \hat{g}_{\textit{GRPO}} + \hat{g}_{\textit{rubric}}
\end{equation}
where
\begin{align}
    \hat{g}_{\textit{GRPO}} &= \frac{1}{n} \sum_{i=1}^{n} A_{\text{base},i} \sum_{t=1}^{T_i} \nabla_\theta \log \pi_\theta(a_t^i \mid s_t^i), \label{eq:g_grpo} \\
    \hat{g}_{\textit{rubric}} &= \frac{1}{n} \sum_{i=1}^{n} \sum_{k=1}^{K_i} \bar{d}_{k,i} \sum_{t \in \mathcal{T}_{k,i}} \nabla_\theta \log \pi_\theta(a_t^i \mid s_t^i). \label{eq:g_rubric}
\end{align}
\end{proposition}

\begin{proof}
The SRaR policy gradient estimate is:
\begin{equation}
    \hat{g}_{\text{SRaR}}
    = \frac{1}{n} \sum_{i=1}^{n} \sum_{t=1}^{T_i} \hat{A}_i^{(t)} \, \nabla_\theta \log \pi_\theta(a_t^i \mid s_t^i).
\end{equation}
Substituting $\hat{A}_i^{(t)} = A_{\text{base},i} + \tilde{r}_i^{(t)}$ and using linearity of summation:
\begin{equation}
    \hat{g}_{\text{SRaR}}
    = \underbrace{\frac{1}{n} \sum_{i=1}^{n} A_{\text{base},i} \sum_{t=1}^{T_i} \nabla_\theta \log \pi_\theta(a_t^i \mid s_t^i)}_{\hat{g}_{\text{GRPO}}}
    + \underbrace{\frac{1}{n} \sum_{i=1}^{n} \sum_{t=1}^{T_i} \tilde{r}_i^{(t)} \, \nabla_\theta \log \pi_\theta(a_t^i \mid s_t^i)}_{\hat{g}_{\text{rubric}}}.
\end{equation}
For the rubric term, since $\tilde{r}_i^{(t)} = \bar{d}_{k,i}$ for all $t \in \mathcal{T}_{k,i}$ (full-span placement) and $\tilde{r}_i^{(t)} = 0$ for tokens not belonging to any step with rubric evaluations, we can rewrite:
\begin{equation}
    \hat{g}_{\text{rubric}}
    = \frac{1}{n} \sum_{i=1}^{n} \sum_{k=1}^{K_i} \bar{d}_{k,i} \sum_{t \in \mathcal{T}_{k,i}} \nabla_\theta \log \pi_\theta(a_t^i \mid s_t^i). \qedhere
\end{equation}
\end{proof}

\begin{remark}[Interpretation]
\label{remark:pg_interp}
Equation~\eqref{eq:g_grpo} is exactly the standard GRPO gradient, which uniformly upweights or downweights the log-probability of the entire response based on outcome correctness. Equation~\eqref{eq:g_rubric} provides a step-selective correction: for each step $k$, the log-probabilities of tokens within that step are scaled by $\bar{d}_{k,i}$, which is positive when the step is relatively better than the same step in other rollouts, and negative otherwise. This allows the gradient to selectively reinforce or suppress individual reasoning steps within an otherwise uniformly scored response.
\end{remark}

\subsection{Variance Analysis}
\label{app:variance}

\begin{proposition}[Scale Alignment Between Outcome and Rubric Signals]
\label{prop:scale}
After normalization, both the outcome advantage and the step-wise rubric signals have controlled second moments:
\begin{align}
    \frac{1}{n}\sum_{i=1}^{n} A_{\text{base},i}^2 &= 1, \label{eq:outcome_var} \\
    \frac{1}{|\mathcal{G}_k|}\sum_{i \in \mathcal{G}_k} \bar{d}_{k,i}^2 &= 1, \quad \forall\, k. \label{eq:rubric_var}
\end{align}
\end{proposition}

\begin{proof}
For the outcome advantage (Eq.~\ref{eq:outcome_var}), using population variance $\sigma_{\text{base}}^2 = \frac{1}{n}\sum_i(r_{\text{base},i} - \mu_{\text{base}})^2$:
\begin{equation}
    \frac{1}{n}\sum_{i=1}^{n} A_{\text{base},i}^2
    = \frac{1}{n\,\sigma_{\text{base}}^2} \sum_{i=1}^{n} (r_{\text{base},i} - \mu_{\text{base}})^2
    = \frac{n\,\sigma_{\text{base}}^2}{n\,\sigma_{\text{base}}^2}
    = 1.
\end{equation}

For the step-wise rubric signal (Eq.~\ref{eq:rubric_var}), using population variance $\sigma_k^2 = \frac{1}{|\mathcal{G}_k|}\sum_{i \in \mathcal{G}_k}(d_{k,i} - \mu_k)^2$:
\begin{equation}
    \frac{1}{|\mathcal{G}_k|}\sum_{i \in \mathcal{G}_k} \bar{d}_{k,i}^2
    = \frac{1}{|\mathcal{G}_k|\,\sigma_k^2}\sum_{i \in \mathcal{G}_k}(d_{k,i} - \mu_k)^2
    = \frac{|\mathcal{G}_k|\,\sigma_k^2}{|\mathcal{G}_k|\,\sigma_k^2}
    = 1.
\end{equation}

Both quantities equal 1, confirming that the two signals operate on exactly the same scale.
\end{proof}

\begin{remark}[Variance Reduction via Step-wise Normalization]
\label{remark:var_reduction}
Without step-wise normalization (e.g., reward pattern 1 or 2), the raw rubric deltas $d_{k,i}$ may have highly variable magnitudes depending on the number and type of rubric items. In the worst case, the rubric signal can dominate or be negligible relative to the outcome advantage, destabilizing training. Step-wise normalization guarantees unit population variance per step, preventing any single step from contributing disproportionately to the gradient. Since the outcome advantage has unit population variance as well (Proposition~\ref{prop:scale}), the two components of the SRaR advantage remain on the same scale throughout training, regardless of the raw rubric reward magnitudes.
\end{remark}

\subsection{Unbiasedness of the Policy Gradient Estimator}
\label{app:unbiased}

\begin{proposition}[SRaR Preserves Unbiasedness of the REINFORCE Estimator]
\label{prop:unbiased}
Define the SRaR reward for rollout $i$ at token $t$ as $R_i^{(t)} = r_{\text{base},i} + \tilde{r}_i^{(t)}$, where $\tilde{r}_i^{(t)}$ is the rubric reward at token $t$. Then the SRaR policy gradient estimator
\begin{equation}
    \hat{g}_{\textit{SRaR}} = \frac{1}{n}\sum_{i=1}^{n}\sum_{t=1}^{T_i} \left( R_i^{(t)} - b_i^{(t)} \right) \nabla_\theta \log \pi_\theta(a_t^i \mid s_t^i)
\end{equation}
is an unbiased estimator of $\nabla_\theta J_{\textit{SRaR}}(\theta)$ for any baseline $b_i^{(t)}$ that is independent of the action $a_t^i$, where $J_{\textit{SRaR}}(\theta) = \mathbb{E}_{y \sim \pi_\theta}[R(y)]$ is the expected SRaR reward.
\end{proposition}

\begin{proof}
This follows from the standard REINFORCE theorem~\citep{williams1992simple}. The rubric reward $\tilde{r}_i^{(t)}$ is a function of the complete response $y_i$ (evaluated by the judge after generation), thus it is a valid return signal. For any baseline $b_i^{(t)}$ that does not depend on $a_t^i$:
\begin{equation}
    \mathbb{E}_{a_t^i \sim \pi_\theta}\!\left[ b_i^{(t)} \nabla_\theta \log \pi_\theta(a_t^i \mid s_t^i) \right]
    = b_i^{(t)} \nabla_\theta \sum_{a} \pi_\theta(a \mid s_t^i)
    = b_i^{(t)} \nabla_\theta 1
    = 0.
\end{equation}
Therefore subtracting the baseline does not introduce bias.

In the SRaR decoupled advantage, the group-level statistics $\mu_{\text{base}}$, $\sigma_{\text{base}}$, $\mu_k$, and $\sigma_k$ serve as components of the baseline. In the leave-one-out interpretation, these statistics computed from the other $n-1$ rollouts are independent of $a_t^i$ for any given rollout $i$, preserving unbiasedness. When computed including rollout $i$ (as in the standard GRPO implementation), the resulting bias is $O(1/n)$ and vanishes as the group size grows, consistent with the analysis of GRPO
\end{proof}


\section{Dataset Statistics}
\label{app:dataset-stats}

Our training set consists of 16{,}560 math problems, each annotated with a structured rubric composed of four tag types: \texttt{SUGGEST} (key solution steps), \texttt{PITFALL} (common errors), \texttt{BONUS} (alternative approaches or insights), and \texttt{ANSWER} (final answer verification). We provide a detailed statistical analysis below.

\subsection{Overall Summary}

\begin{table}[H]
\centering
\caption{Overall dataset statistics.}
\label{tab:overall}
\begin{tabular}{lr}
\toprule
\textbf{Statistic} & \textbf{Value} \\
\midrule
Total samples & 16{,}560 \\
Mean rubric items per sample & 7.53 \\
Median rubric items per sample & 8 \\
Mean words per rubric item & 18.5 \\
\bottomrule
\end{tabular}
\end{table}

\subsection{Rubric Item Counts by Type}

\begin{table}[H]
\centering
\caption{Statistics of rubric item counts by tag type.}
\label{tab:rubric-type-stats}
\begin{tabular}{lcccccc}
\toprule
\textbf{Tag Type} & \textbf{Mean} & \textbf{Std} & \textbf{Median} & \textbf{Min} & \textbf{Max} & \textbf{Avg.\ Ratio} \\
\midrule
\textsc{Suggest} & 4.59 & 0.86 & 5 & 2 & 8 & 60.9\% \\
\textsc{Pitfall} & 1.26 & 0.51 & 1 & 0 & 3 & 16.7\% \\
\textsc{Bonus}   & 0.67 & 0.47 & 1 & 0 & 2 & 8.9\% \\
\textsc{Answer}  & 1.00 & 0.00 & 1 & 1 & 1 & 13.3\% \\
\midrule
\textbf{Total}   & 7.53 & 0.79 & 8 & 4 & 12 & 100\% \\
\bottomrule
\end{tabular}
\end{table}

\subsection{Distribution of Total Rubric Item Count}

Figure~\ref{fig:rubric-count-dist} shows the distribution of total rubric item count per sample.

\begin{figure}[H]
\centering
\begin{tikzpicture}
\begin{axis}[
    ybar,
    bar width=16pt,
    width=0.75\textwidth,
    height=0.4\textwidth,
    xlabel={Number of rubric items},
    ylabel={Number of samples},
    ymin=0,
    ymax=10500,
    xtick={4,5,6,7,8,9,10,11,12},
    nodes near coords,
    nodes near coords align={vertical},
    every node near coord/.append style={font=\scriptsize},
    enlarge x limits=0.1,
    scaled y ticks=false,
    yticklabel style={/pgf/number format/fixed, /pgf/number format/1000 sep={,}},
]
\addplot coordinates {
    (4, 13)
    (5, 175)
    (6, 1414)
    (7, 5449)
    (8, 8565)
    (9, 873)
    (10, 67)
    (11, 3)
    (12, 1)
};
\end{axis}
\end{tikzpicture}
\caption{Distribution of total rubric item count per sample. The majority of samples (84.6\%) contain 7--8 rubric items.}
\label{fig:rubric-count-dist}
\end{figure}
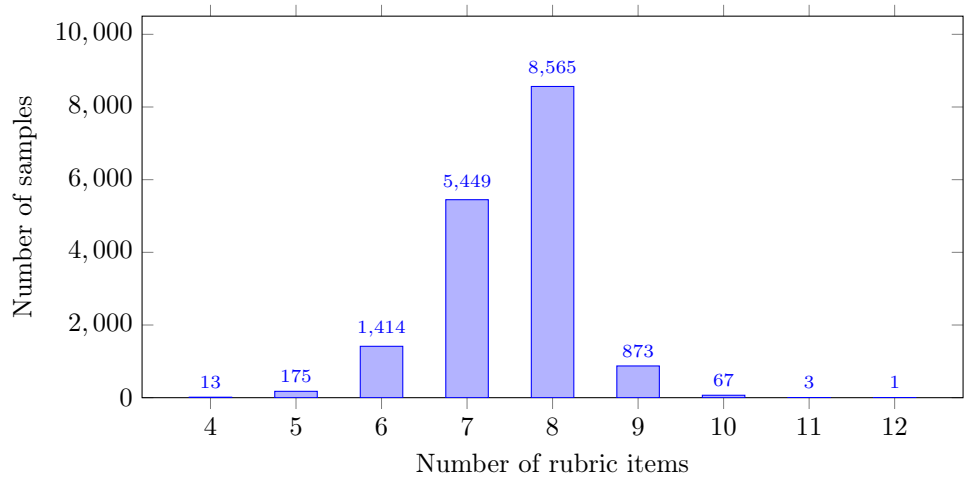

\subsection{Distribution of \textsc{Suggest} Count}

Figure~\ref{fig:suggest-dist} shows the distribution of \textsc{Suggest} items per sample.

\begin{figure}[H]
\centering
\begin{tikzpicture}
\begin{axis}[
    ybar,
    bar width=16pt,
    width=0.75\textwidth,
    height=0.4\textwidth,
    xlabel={Number of \textsc{Suggest} items},
    ylabel={Number of samples},
    ymin=0,
    ymax=8500,
    xtick={2,3,4,5,6,7,8},
    nodes near coords,
    nodes near coords align={vertical},
    every node near coord/.append style={font=\scriptsize},
    enlarge x limits=0.12,
    scaled y ticks=false,
    yticklabel style={/pgf/number format/fixed, /pgf/number format/1000 sep={,}},
]
\addplot coordinates {
    (2, 79)
    (3, 1327)
    (4, 6208)
    (5, 6778)
    (6, 1970)
    (7, 197)
    (8, 1)
};
\end{axis}
\end{tikzpicture}
\caption{Distribution of \textsc{Suggest} items per sample. Most samples contain 4--5 solution steps.}
\label{fig:suggest-dist}
\end{figure}
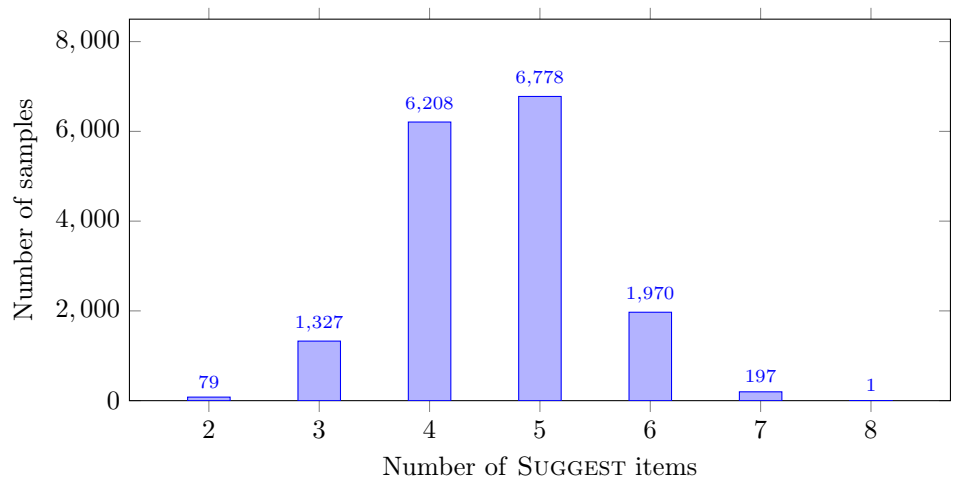

\subsection{Distribution of \textsc{Pitfall} Count}

Figure~\ref{fig:pitfall-dist} shows the distribution of \textsc{Pitfall} items per sample.

\begin{figure}[H]
\centering
\begin{tikzpicture}
\begin{axis}[
    ybar,
    bar width=22pt,
    width=0.55\textwidth,
    height=0.4\textwidth,
    xlabel={Number of \textsc{Pitfall} items},
    ylabel={Number of samples},
    ymin=0,
    ymax=14000,
    xtick={0,1,2,3},
    nodes near coords,
    nodes near coords align={vertical},
    every node near coord/.append style={font=\scriptsize},
    enlarge x limits=0.2,
    scaled y ticks=false,
    yticklabel style={/pgf/number format/fixed, /pgf/number format/1000 sep={,}},
]
\addplot coordinates {
    (0, 439)
    (1, 11417)
    (2, 4624)
    (3, 80)
};
\end{axis}
\end{tikzpicture}
\caption{Distribution of \textsc{Pitfall} items per sample. 97.3\% of samples have 1--2 pitfall warnings.}
\label{fig:pitfall-dist}
\end{figure}
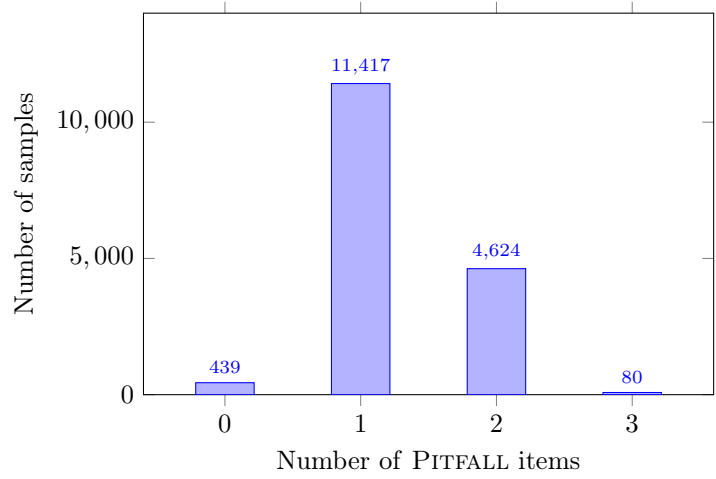

\subsection{Distribution of \textsc{Bonus} Count}

Figure~\ref{fig:bonus-dist} shows the distribution of \textsc{Bonus} items per sample.

\begin{figure}[H]
\centering
\begin{tikzpicture}
\begin{axis}[
    ybar,
    bar width=22pt,
    width=0.55\textwidth,
    height=0.4\textwidth,
    xlabel={Number of \textsc{Bonus} items},
    ylabel={Number of samples},
    ymin=0,
    ymax=13500,
    xtick={0,1,2},
    nodes near coords,
    nodes near coords align={vertical},
    every node near coord/.append style={font=\scriptsize},
    enlarge x limits=0.3,
    scaled y ticks=false,
    yticklabel style={/pgf/number format/fixed, /pgf/number format/1000 sep={,}},
]
\addplot coordinates {
    (0, 5449)
    (1, 11109)
    (2, 2)
};
\end{axis}
\end{tikzpicture}
\caption{Distribution of \textsc{Bonus} items per sample. 67.1\% of samples include one bonus insight.}
\label{fig:bonus-dist}
\end{figure}
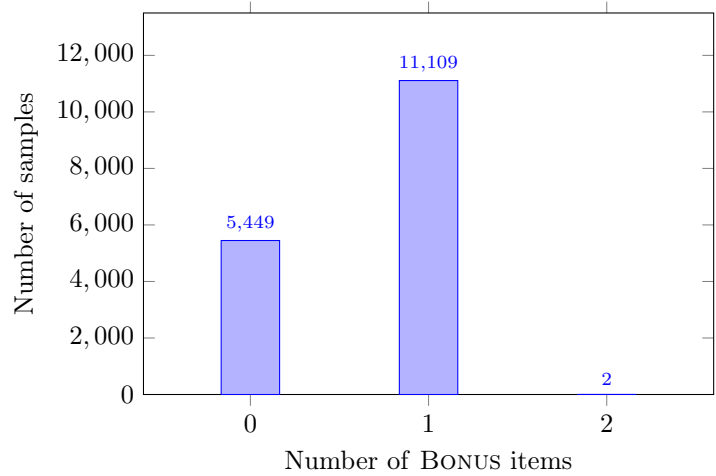

\subsection{Co-occurrence of \textsc{Pitfall} and \textsc{Bonus}}

\begin{table}[H]
\centering
\caption{Co-occurrence matrix of \textsc{Pitfall} and \textsc{Bonus} item counts (\# of samples).}
\label{tab:cooccurrence}
\begin{tabular}{lrrr|r}
\toprule
 & \textbf{Bonus=0} & \textbf{Bonus=1} & \textbf{Bonus=2} & \textbf{Total} \\
\midrule
\textbf{Pitfall=0} & 369 & 70 & 0 & 439 \\
\textbf{Pitfall=1} & 3{,}416 & 8{,}000 & 1 & 11{,}417 \\
\textbf{Pitfall=2} & 1{,}617 & 3{,}006 & 1 & 4{,}624 \\
\textbf{Pitfall=3} & 47 & 33 & 0 & 80 \\
\midrule
\textbf{Total} & 5{,}449 & 11{,}109 & 2 & 16{,}560 \\
\bottomrule
\end{tabular}
\end{table}


\nocite{xu2026grouter, zhao2025mentor, chen2026learning, chen2026snapbench, jin2025dynamic, li2026memoryathena, li2026frozen, xiao2026adapting, liu2026driveva, liu2026universe}

\newpage